%% file: main-final.tex
\documentclass[journal]{style/IEEEtran}
\usepackage[cmex10]{amsmath}
\usepackage{amssymb,  amsfonts, amsthm, cite, algorithm, color, graphicx, hyperref, algorithm, algorithmic, array, verbatim}
\usepackage{style/mysymbol}
\usepackage[english]{babel}

\usepackage[inline,shortlabels]{enumitem}
\usepackage{enumitem}
\newlist{inlinelist}{enumerate*}{1}
\setlist*[inlinelist,1]{%
  label=(\roman*),
}

\newtheorem{theorem}{Theorem}[section]
\newtheorem{lemma}[theorem]{Lemma}
\newtheorem{prop}[theorem]{Proposition}
\newtheorem{cor}[theorem]{Corollary}

\newtheorem{assumption}[theorem]{Assumption}

\def\dist{\mathrm{dist}}

\def\a{\alpha}
\def\b{\beta}

\def\o{\omega}

\def\Gc{\mathcal{G}}
\def\Vc{\mathcal{V}}
\def\Ec{\mathcal{E}}

\def\Xc{\mathcal{X}}
\def\Fc{\mathcal{F}}
\def\Nc{\mathcal{N}}
\def\Es{\mathsf{E}}

\def\1b{\mathbf{1}}
\def\0b{\mathbf{0}}

\def\la{\langle}
\def\ra{\rangle}
\def\eb{\mathbf{e}}

\def\zb{\mathbf{z}}

\def\vb{\mathbf{v}}

\def\pb{\mathbf{p}}
\def\ab{\mathbf{a}}

\def\diag{\textrm{diag}}

\def\norm#1{{\left\|#1\right\|}}

\def\abs#1{\left|#1\right|}
\def\set#1{\left\{#1\right\}}

\def\tb{\boldsymbol{\theta}}

\def\l{\ell}

\def\k{\kappa}

\def\bd{\boldsymbol{\delta}}

\def\proj{\mathsf{\Pi}}
\def \Sym {{\rm Sym}}

\title{
\input{reviews-response/title}
}
\author{\input{text-16p/authors}}
\begin{document}
\maketitle

\begin{abstract}
\input{text-16p/abstract}
\end{abstract}


\IEEEpeerreviewmaketitle

\section{Introduction}
\label{sec:intro}
\input{text-16p/intro}

\section{Problem Formulation}
\label{sec:pf}
\input{text-16p/pf2}

\section{Concurrent Estimation and Sensor Planning}
\label{sec:sol}
\input{text-16p/sol2}

\section{Main Results}
\label{sec:assume}
\input{text-16p/assume2}

\section{Convergence of Algorithm~\ref{alg:filter+opt}}
\label{sec:converge}
\input{text-16p/converge2}

\section{Numerical Experiments}
\label{sec:sims}
\input{text-16p/sims2}

\section{Conclusion}
\label{sec:conc}

\input{text-16p/conc}

\appendices

\section{Proof of Lemmas~\ref{lem:summable_constraints}~and~\ref{lem:summable_grads}}
\label{app:summability}
\input{text-16p/proofs1-2}

\section{Proof of Lemmas~\ref{lem:basic}-\ref{lem:disagree}}
\label{app:lemmaproofs2}
\input{text-16p/proofs2}




\bibliographystyle{IEEEtran}
\bibliography{charlie-refs}

\begin{IEEEbiography}[{\includegraphics[width=1in,height=1.25in,clip,keepaspectratio]{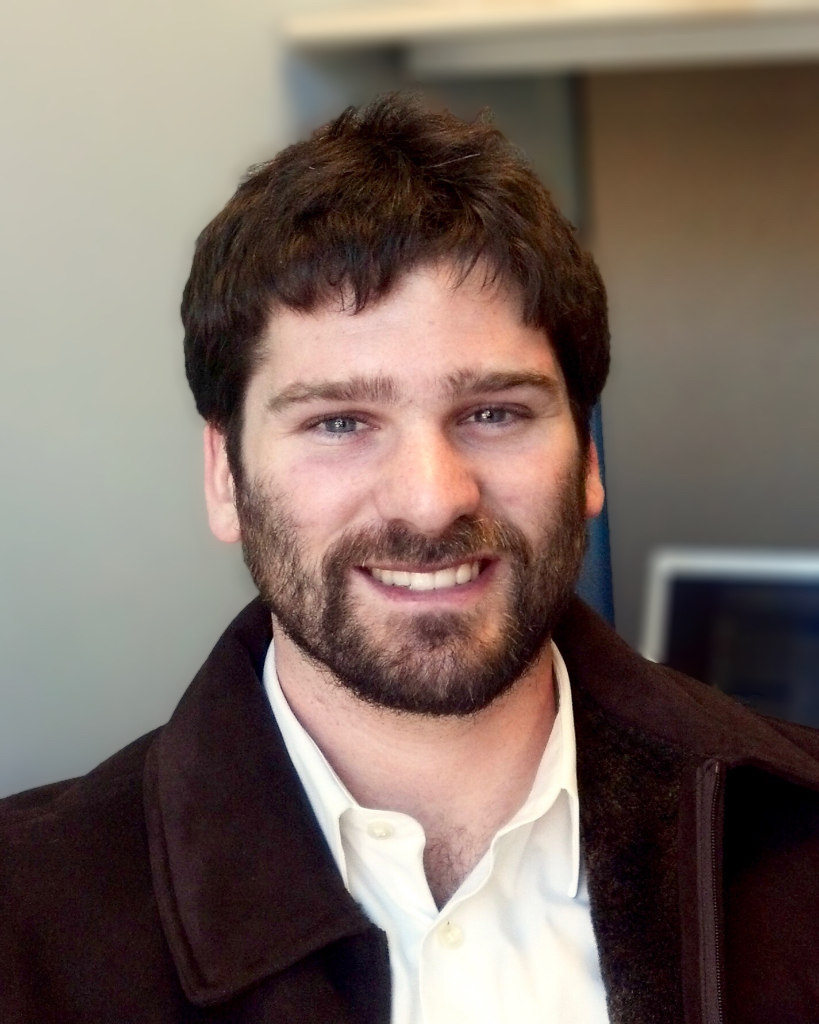}}]%
{Charles Freundlich}
\input{bios/freundlich_bio}
\end{IEEEbiography}
\begin{IEEEbiography}[{\includegraphics[width=1in,height=1.25in,clip,keepaspectratio]{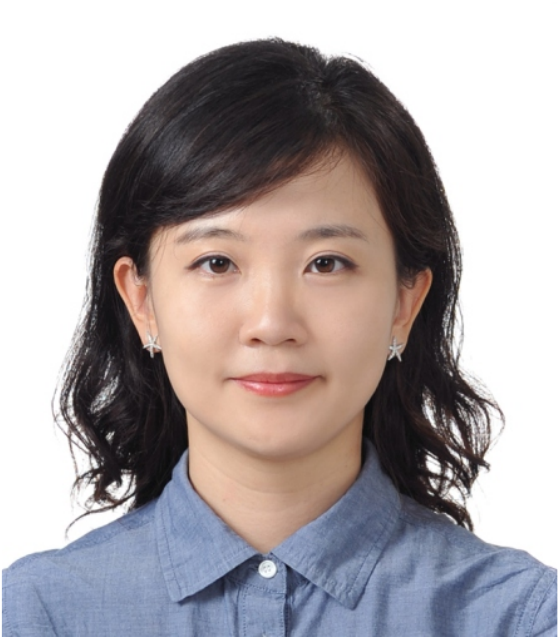}}]%
{Soomin Lee}
\input{bios/soomin_bio}
\end{IEEEbiography}
\begin{IEEEbiography}[{\includegraphics[width=1in,height=1.25in,clip,keepaspectratio]{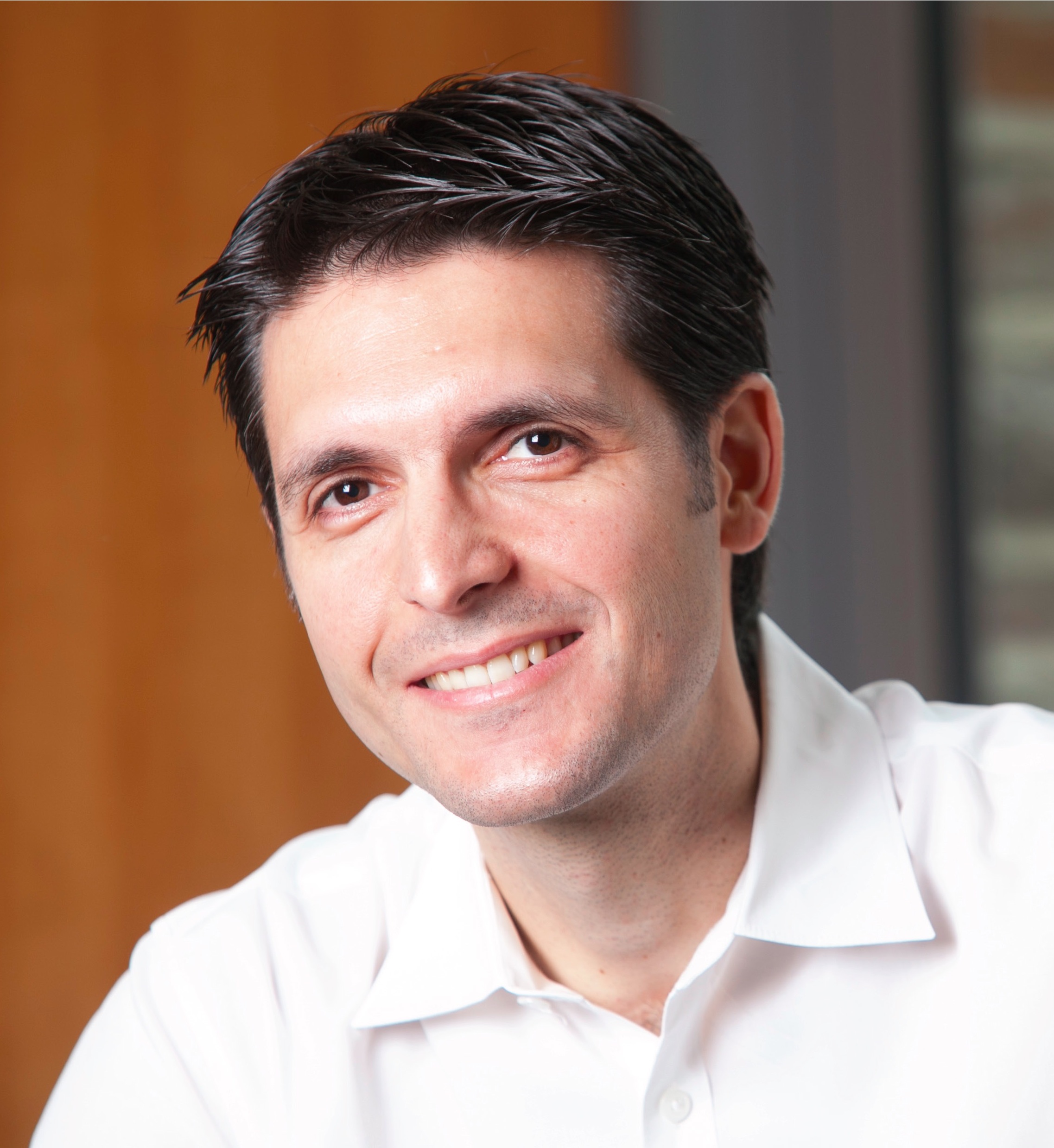}}]%
{Michael Zavlanos}
\input{bios/zavlanos_bio}
\end{IEEEbiography}
\end{document}

%% file: reviews-response/title.tex
Distributed Active State Estimation with User-Specified Accuracy

%% file: text-16p/authors.tex
Charles~Freundlich,~\IEEEmembership{Student Member,~IEEE,}
Soomin~Lee,~\IEEEmembership{Member,~IEEE,}
and~Michael~M.~Zavlanos,~\IEEEmembership{Member,~IEEE}%
\thanks{
Charles Freundlich, Soomin Lee, and Michael M. Zavlanos are with the Dept. of Mechanical Engineering and Materials Science, Duke University, Durham, NC 27708, USA 
\href{mailto:crf29@duke.edu}{\{crf29, }
\href{mailto:s.lee@duke.edu}{s.lee, }
\href{mailto:mz61@duke.edu}{mz61\}}@duke.edu.
This work is supported in part by the NSF award CNS \#1302284 and by ONR under grant \#N000141410479.
}%

%% file: text-16p/abstract.tex
In this paper, we address the problem of controlling a network of mobile sensors so that a set of hidden states are estimated up to a user-specified accuracy.
The sensors take measurements and fuse them online using an Information Consensus Filter (ICF).
At the same time, the local estimates guide the sensors to their next best configuration.
This leads to an LMI-constrained optimization problem that we solve by means of a new distributed random approximate projections method. The new method is robust to the state disagreement errors that exist among the robots as the ICF fuses the collected measurements.
Assuming that the noise corrupting the measurements is zero-mean and Gaussian and that the robots are self localized in the environment, the integrated system converges to the next best positions from where new observations will be taken.
This process is repeated with the robots taking a sequence of observations until the hidden states are estimated up to the desired user-specified accuracy. We present simulations of sparse landmark localization, where the robotic team achieves the desired estimation tolerances while exhibiting interesting emergent behavior.

%% file: text-16p/intro.tex
\IEEEPARstart{I}{n} this paper, we control a robotic sensor network to estimate a collection of hidden states so that desired accuracy thresholds and confidence levels are satisfied.
We require that estimation and control are completely distributed, i.e., belief states and control actions are decided upon locally through pairwise communication among robots in the network.

A common assumption in problems like the one discussed herein is that observations depend linearly on the state and are corrupted by Gaussian noise
as well as that the sensors are self localized
\cite{chung06,jalalkamali12,freundlich13cdc,freundlich15acc, freundlich13icra,vanderhook15, simonetto11,derenick09}.
Under these assumptions, a typical approach to estimate a set of hidden states is to use an Information Filter (IF).
The role of the IF is to fuse new observations of the hidden states with a prior Gaussian distribution to create a minimum-variance posterior Gaussian distribution.
Typically, the observations are subject to error covariance that is a function of, e.g., viewing position, angle of incidence, and the sensing modality.
Several such models have been proposed in the distributed active sensing literature, see, e.g., \cite{chung06, simonetto11, jalalkamali12,  freundlich13cdc, freundlich13icra,vanderhook15,derenick09}, but regardless of the specific model used for the measurement covariance, the IF combines the data among the nodes in the network in an additive way, making the IF algorithm readily distributed; see, e.g., Information Consensus Filtering (ICF) \cite{kamal12}.

The dependence of the measurement model on the sensor state has been widely used to obtain planning algorithms that, given a history of measurements, determine the next best set of observations of a hidden state.
The objective is that this new set of observations optimizes an information-theoretic criterion of interest.
In this paper, we choose to maximize the minimum eigenvalue of the posterior information matrix, effectively minimizing the directional variance of every hidden state.
We express this requirement by reformulating the problem as a Linear Matrix Inequality (LMI) constrained optimization problem. Specifically, we introduce as many LMI constraints as the number of hidden states, that are coupled with respect to the sensors via an ICF that estimates the uncertainty of each hidden state using the sensor measurements.
To solve this problem, we propose a new distributed optimization algorithm, which we call \emph{random approximate projections}, that 
is robust to the state disagreement errors that exist among the robots as an ICF fuses the collected measurements.
The ability to handle such errors is critical in combining distributed estimation and planning in one process.

The proposed distributed random approximate projections method falls in the class of so called consensus-based optimization algorithms, where the goal is to interleave a decentralized averaging step between local optimization steps so that the local estimates of all agents simultaneously obtain a network-wide consensus and optimality of the global problem.  
Existing algorithms require expensive optimization steps or projections on a complicated constraint set, e.g., sets of LMI constraints as in this paper, at every iteration.
Instead, our method employs a subgradient step in the direction that minimizes the violation of randomly selected local constraints, significantly reducing computational cost compared to relevant methods that rely on projection operations.
Assuming that each LMI constraint is selected with non-zero probability, we show that our proposed algorithm converges almost surely to the next best sensor positions from where a new set of observations minimizes the worst-case uncertainty of the hidden states.
This process is repeated with the sensors taking a sequence of observations
until all the hidden states are estimated up to the desired user-specified
accuracy.
For details on relevant distributed optimization methods, we refer the interested reader to \cite{nedic15} and references therein.

\input{reviews-response/markov-incremental}
\input{reviews-response/dual}

\input{reviews-response/applications}
A common characteristic of the applications discussed above is that the sensors must collectively reason over the possible posterior distributions resulting from 
\begin{inlinelist}
\item multiple hidden states being observed by a single sensor or
\item individual hidden states being observed by several sensors at once.
\end{inlinelist}
%
%
%
%
%
%
%
Without assuming Gaussian distributions, optimization over the possible posteriors can only be done approximately \cite{julian12}.
In this paper, and in much of the relevant literature \cite{chung06,jalalkamali12,freundlich13cdc,freundlich15acc, freundlich13icra,vanderhook15}, it is assumed that the posteriors for the hidden states are Gaussian.
Under this assumption, typical choices of an objective function to minimize in order to obtain the next best set of observations include the trace, the determinant, or the maximum eigenvalue of the posterior covariance matrix.
Intuitively, the trace minimizes the sum of the uncertainties of the hidden states, the determinant the volume of the confidence ellipsoids of the hidden states, and the maximum eigenvalue the worst-case error.
\input{reviews-response/worst-case}
%

Methods that account for nonlinear dynamics and arbitrary error distributions \cite{spaan14,huang15} are also relevant to the problem addressed herein.
In these cases, updates to the hidden state pdfs may no longer be available in closed form, and so numerical methods are needed to evaluate the expectation integrals.
To account for uncertainty in the sensor location and actions, decentralized Partially Observable Markov Decision Process (dec-POMDP) can be used to solve these problems.
Dec-POMDPs have coupled value functions among the agents, and typically require a centralized planner \cite{omidshafiei15}.
At the same time, comparing the possible posterior distributions can be computationally expensive, so these methods do not usually scale well with the number of targets and agents.
Although many of these works assume that the robots are self localized with respect to a common global reference frame
\cite{huang15,julian12,vanderhook15, chung06, jalalkamali12, freundlich13cdc,freundlich15acc}, 
this is certainly not always the case; see, e.g., work on partially observable systems \cite{omidshafiei15,spaan14} or the SLAM problem \cite{nerurkar14}.

{\bf Contributions:}
In this paper we propose a distributed estimation and sensor planning method that can ensure desired estimation tolerances for large numbers of hidden states.
To our knowledge, even though the distributed active sensing literature is well-developed, the ability to control worst-case estimation uncertainty in a decentralized way is new.
This is possible by combining an ICF for distributed data fusion with an LMI-constrained optimization problem for sensor planning, for which we propose a new computationally inexpensive distributed method that is robust to the inexact pre-consensus filter data.
Coordination among the sensors requires only weak connectivity of the communication graph and can, thus, handle occasional disruption of communication links.

The distributed planning algorithm itself is an extension of the authors' previous work \cite{lee15distributed} and \cite{lee15approx}.
The difference is that the method developed in this paper is robust to state disagreement errors that appear as a result of the information filter.
The ability to handle such errors is critical in integrating distributed estimation and planning.
To the best of our knowledge, there are no distributed optimization methods that can handle efficiently LMI constraints under inexact data.
It is also worth mentioning that a wide variety of convex constraints arising in control, such as quadratic inequalities, inequalities involving matrix norms, Lyapunov inequalities and quadratic matrix inequalities, can be cast as LMIs.
Therefore, our developed algorithm can be used for other, potentially unrelated, control problems involving inexact data.

The rest of the paper is organized as follows.
Section~\ref{sec:pf} formulates the problem under consideration.
Then, Section~\ref{sec:sol} develops the proposed distributed method to maximize the minimum eigenvalue of the information matrix.
Section \ref{sec:assume} states assumptions and the main convergence result of our proposed algorithm.
Section \ref{sec:converge} shows that the proposed algorithm converges to the optimal point.
Section~\ref{sec:sims} provides simulations, and
Section~\ref{sec:conc} concludes the paper.

%% file: reviews-response/markov-incremental.tex
Related to our random projection technique is the Markov incremental random subgradient method \cite{johansson09}.
The key difference is in the mechanism that is used to distribute the computations. 
In the Markov incremental algorithm, the agents sequentially update a single iterate sequence by passing the iterate to each other following a time inhomogeneous Markov chain.
In our algorithm, each agent maintains a local iterate sequence and updates this by communicating with its neighbors.

%% file: reviews-response/dual.tex
In two recent papers \cite{jamali15,simonetto16}, a dual-decomposition method is proposed that can be used for problems similar to the distributed control problem considered here, with lower communication overhead and increased privacy.
Generally, dual methods are well suited for network optimization problems, they come at increased local cost because of the need to compute the dual function at each node.
The approaches described in \cite{jamali15,simonetto16} additionally require a special problem structure; specifically, a decomposable objective and/or a star network topology.

%% file: reviews-response/applications.tex
Applications in the area of distributed sensor planning and estimation that are directly relevant to our work include 
SLAM \cite{nerurkar14},
localization \cite{freundlich15acc, freundlich13icra,vanderhook15}, coverage \cite{simonetto11, derenick09}, 
mobile target tracking \cite{huang15,freundlich13cdc,jalalkamali12,chung06}, classification \cite{julian12}, and inverse problems for PDE-constrained systems \cite{patan12, khodayi15}. 
Compared to state estimation, inverse problems for PDE-constrained systems address both state and source identification and are particularly difficult to solve.

%% file: reviews-response/worst-case.tex
To our knowledge, in the context of the distributed active localization problem, existing works that control the maximum eigenvalue of the error covariance matrix either 
focus on sensor selection from a discrete set \cite{JoshiBoyd09, ranieri14, chepuri14, jamali15,simonetto16},
are limited to two sensors \cite{vanderhook15}, or
do not factor in simultaneous measurements by other sensors when deciding where to go \cite{patan09,jalalkamali12,freundlich15acc}.

%% file: text-16p/pf2.tex
Consider the problem of estimating $m$ stationary hidden states $\set{\bbx_i \in \reals^p }_{i \in \ccalI}$ using noisy observations from $n$ mobile sensors, where $p$ is the dimension of the hidden states and $\ccalI = \set{1, \dots, m}$. 
We assume that the hidden states are sparse so that there are no correlations between them.
Denote the locations of the mobile sensors at time $t$ by $\set{ \bbr_s(t) \in \reals^q}_{s \in \ccalS},$ where $q$ is the dimension of the sensors' configuration space and $\ccalS = \set{1, \dots, n}.$ The  observation of state $i$ by sensor $s$ at time $t$ is given by
\label{eq:linsys}
\begin{align}
\bby_{i,s} (t) &= \bbx_i + \bbzeta_{i,s}(t).\label{eq:obs}
\end{align}
We assume that the measurement errors are Normally distributed so that $\bbzeta_{i,s}(t) \sim N \big(\bb0, Q (\bbr_s (t) , \bbx_i ) \big)$, where $Q \colon \reals^{q + p} \to \Sym_{+}(p, \reals)$ denotes the measurement precision matrix, or information matrix (not the covariance), and $\Sym_{+}(p, \reals)$ is the set of $p \times p$ symmetric positive semidefinite real matrices. We also assume that if signal $\bbx_i$ is out of range of sensor $s$, then the function $Q$ returns the information matrix  $\bb0_{p \times p},$ corresponding to infinite variance.
We provide an example of the measurement information function $Q$ in Section~\ref{sec:sims}.
Denote the set of all observations at time $t$ by
$
\ccalO(t)\triangleq \set{\left( \bby_{i,s} (t) , Q (\bbr_s(t), \bbx_i) \right)}_{ i \in \ccalI, s \in \ccalS}
$

Moreover, let $\hbbx_{i} (t)$ denote the estimate of $\bbx_i$ at time $t$ and $S_{i} (t)$ denote the information matrix corresponding to the estimate $\hbbx_{i}(t).$
\input{reviews-response/initialize-x}
Let also $\hbbx (t)\in  \reals^{pm}$ denote the stack vector of all estimates $\hbbx_{i}(t) $ and $S(t) \in \Sym_{+} (p, \reals)\times\dots\times\Sym_{+} (p, \reals) :=\left( \Sym_{+} (p, \reals) \right)^m$ denote the collection of all information matrices $S_{i}(t)$ at time $t$. 
Then, the global data at time $t$ can be captured by the set
$
\ccalD (t) \triangleq (\hbbx  (t),S (t))  \in \reals^{pm} \times \left( \Sym_{+} (p, \reals) \right)^m
.
$

Given prior data $ \ccalD (t-1)$ known by all sensors, and observations $\ccalO(t-1)$, the current set of data $\ccalD(t)$ can be determined in a distributed way using an Information Consensus Filter (ICF) \cite{kamal12} as\footnote{Explicit details on the mapping \eqref{eq:if-map} are provided in Section~\ref{sec:sol}.}
\begin{equation}\label{eq:if-map}
\left( \ccalD (t-1) , \ccalO(t-1)\right) \stackrel{{\rm  ICF}}{\longmapsto} \ccalD (t).
\end{equation}
Then, the goal is to control the next set of observations $\ccalO(t)$ in order to reduce the future uncertainty in $\ccalD(t+1)$. We can achieve this goal by controlling the mobile sensors and, therefore, the expected information $\set{ Q(\bbr_s(t), \bbx_i (t) }_{i \in \ccalI, s \in \ccalS}$ associated with the observations $\set{\bby_{i,s}(t)}_{i \in \ccalI, s \in \ccalS}$.
Specifically, let $\bbr_s(t) = \bbr_s(t - 1) + \bbu_s$ denote the motion model of sensor $s$, where $\bbu_s \in \reals^q$ is a candidate control input used to drive sensor $s$ to a new position $\bbr_s(t)$.
Then, 
our goal is to maximize the minimum eigenvalues of
the information matrices in $\ccalD(t+1)$ that, for each hidden state $i,$ are given by
\begin{equation}
\label{eq:post-info}
S_{i} (t+1) \! = \!\!\!\overbrace{S_{i} (t) }^\text{from $\ccalD(t)$} \!\!\!\!+ \underbrace{   \sum\nolimits_{s \in\ccalS} Q(\bbr_s(t-1) + \bbu_s, \overbrace{ \hbbx_i (t)}^\text{from $\ccalD(t)$} )}_\text{from $\ccalO(t)$}.
\end{equation}
We assume that $\bbu_s$ lies in a compact set containing the origin.
Denote this set by $\ccalU_s.$ 
Note that there is a $\delta>0$ such that the closed ball of radius $\delta$ contains $\ccalU_s.$
Define $\ccalU \triangleq \prod_{s \in \ccalS} \ccalU_s$ and let $\bbu \in \reals^{qn}$ denote the stack of all controllers $\bbu_s.$
Additionally, define the set $\Gamma \triangleq \prod_{i \in \ccalI} [0, \tau_i]$, where $\tau_i$ represent the given user-specified estimation thresholds for each hidden state $i$. 
Define also $\bbgamma= (\gamma_1 , \dots, \gamma_m) \in \reals^m.$
With the new notation, the problem we consider in this paper can be defined as
\begin{subequations}
\label{eq:p1}
\begin{align}
&\max_{
(\bbgamma, \bbu)\in \Gamma \times \ccalU
} \quad \sum\nolimits_{i \in \ccalI}  \gamma_i  
\\
&\textrm{s.t} \quad  
\gamma_iI \preceq S_{i}+  \sum\nolimits_{s \in\ccalS} Q(\bbr_s + \bbu_s, \hbbx_i) \;\;  \forall i \in \ccalI \label{eq:nonlin_constraint},
\end{align}
\end{subequations}
where we have dropped dependence on the time $t$ for notational convenience. For the explicit references to time in \eqref{eq:nonlin_constraint} see \eqref{eq:post-info}.
The symbol $\preceq$ denotes an ordering on the negative semidefinite cone, i.e., $A \preceq B$ if and only if  $A - B$ is negative semidefinite. The great challenge in solving \eqref{eq:p1} is that the constraints \eqref{eq:nonlin_constraint} require knowledge of $\ccalD(t)$, however, only $\ccalD(t-1)$ and those parts of $\ccalO(t-1)$ corresponding to local sensor measurements are available to any given sensor. In what follows, we propose an integrated approach that combines the ICF to determine the global data $\ccalD(t)$ and the solution of the optimization \eqref{eq:p1} in one process. We show that our proposed method can handle pre-consensus disagreement errors due to the ICF and converges to the next best set of observations $\ccalO(t)$, as desired.

%% file: reviews-response/initialize-x.tex
%
We assume that initial values for the pair $\left( \hbbx_{i} (0), S_{i} (0) \right)$ are available {\em a priori}.

%% file: text-16p/sol2.tex
Problem~\eqref{eq:p1} is a nonlinear semidefinite program.
To simplify this problem, we replace the nonlinear function $Q ( \bbr_s + \bbu_s ,  \hbbx_i)$ in \eqref{eq:nonlin_constraint} with its Taylor expansion about the input $\bbu_s =0$.
Specifically, define the function 
\[
h \colon \! \reals^{m + qn} \!\times \left(\reals^{pm}\! \times \! \left( \Sym_{+} (p, \reals) \right)^m \right)\! \times \! \ccalI \to \Sym (p, \reals),
\]
where $\Sym(p, \reals)$ denotes the set of symmetric $p \times p$ dimensional matrices,
by 
\begin{align}\label{eq:define_h}
h  \left((\bbgamma, \bbu); \ccalD ,  i\right)&\triangleq
\gamma_i I -
S_{i}
\\
-&
\sum_{s \in\ccalS}  
\bigg[ 
Q(\bbr_s , \hbbx_{i}) 
+
\sum\nolimits_{j=1}^q
\nabla_j Q(\bbr_s, \hbbx_i ) [ \bbu_s ]_j
\bigg]. \nonumber
\end{align}
The notation
$\nabla_j Q(\bbr_s, \hbbx_i  )$ is the partial derivative of the matrix-valued function $Q$ with respect to the $j$-th coordinate of $\reals^{q + p},$ evaluated at $(\bbr_s, \hbbx_i  )$.
Also, $[ \bbu_s ]_j$ denotes the $j$-th coordinate of the vector $ \bbu_s$, and the semicolon notation in the arguments of $h$ separates the decision variables $(\bbgamma, \bbu)$ from the parameters $\ccalD$ and $i$.
Note that the value of $h$ is negative semidefinite at a point $\left((\bbgamma, \bbu); \ccalD , i\right)$ if and only if the linearized version of the constraint \eqref{eq:nonlin_constraint} is satisfied for $\gamma_i$ using the data $\ccalD = (\hbbx, S).$
The effect of the linear approximation is investigated in Section~\ref{sec:sims}, Fig.~\ref{f:viols}.

To simplify notation, collect all decision variables in the vector
$\bbz = \left(\bbgamma,  \bbu\right) \in \reals^{m + qn}$. Then, to decompose the linearized version of \eqref{eq:p1} over the set of sensors so that the problem can be solved in a distributed way, let $\bbz_s$ denote a copy of $\bbz$ that is local to sensor $s \in \ccalS.$
Additionally, let
$
f (\bbz_s) = - \sum_{i \in \ccalI} \gamma_{i,s}$ so that we may express the global objective as $\sum_{s \in \ccalS}f (\bbz_s).$
Then, problem~\eqref{eq:p1} with linearized constraints can be expressed as
\begin{subequations}
\label{eq:distributed}
\begin{align}
\min_{
\set{\bbz_s}_{s \in \ccalS}\subset \ccalX_0
} \quad &\sum\nolimits_{s \in \ccalS} f(\bbz_s) \label{eq:sep_obj}
\\
\textrm{s.t} \quad  h\left(\bbz_s; \ccalD, i \right) &\preceq \bb0\quad \forall i \in \ccalI, \forall s \in \ccalS \label{eq:LMI2},
\end{align}
\end{subequations}
where $\ccalX_0 = \Gamma \times \ccalU$. In problem \eqref{eq:distributed}, the objective function \eqref{eq:sep_obj} is separable among the sensors and the constraints \eqref{eq:LMI2} are linear in $\bbz_s$ and local to every sensor.
In this form, problem \eqref{eq:distributed} can be decomposed and solved in a distributed way.
A necessary requirement is that, at the solution, the local variables $\bbz_s$ need to agree for all sensors $s,$ i.e., $\bbz_s = \bbz_v$ for all $s,v \in \ccalS$.
As discussed in Section \ref{sec:pf}, the main challenge in solving problem \eqref{eq:distributed} is that 
\input{reviews-response/sep}

\begin{algorithm}[t]
\caption{Concurrent Filtering and Optimization for $s \in \ccalS$ during the interval $[t-1, t]$}
\begin{algorithmic}[1]
\label{alg:filter+opt}
\REQUIRE Data $\ccalD_s(t-1)$, measurements $\set{\bby_{i,s}(t-1) }_{i \in \ccalI}$, and $\bbz_{s,0}$ feasible for \eqref{eq:distributed}.
\STATE $(\hbbx,S) \gets \ccalD_s(t-1)$
\STATE $\Xi_{i,s,0} \gets \frac{1}{n}S_i +  Q(\bbr_s, \hbbx_i)$ for all $i \in \ccalI$
\label{line:Xi}
\STATE $\bbxi_{i,s,0}\gets \left( \frac{1}{n} S _i + Q(\bbr_s, \hbbx_i)\right) \bby_{i,s} $ for all $i \in \ccalI$
\label{line:xi}
\FOR{$k=1, 2, \dots$}
\label{line:loop}
\STATE Broadcast  $\bbz_{s,k-1}, \set{\Xi_{i, s, k-1}, \bbxi_{i, s, k-1} \mid i \in \ccalI}$ to $\ccalN_{s,k}$
\label{line:cast}
\STATE $\bbp_{s,k}\gets \sum_{j \in \mathcal{N}_{s,k}} [W_k]_{sj} \bbz_{s,k-1}$
\label{line:p}
\STATE $\Xi_{i, s, k}\gets  \sum_{j \in \mathcal{N}_{s,k}} [W_k]_{sj} \Xi_{i, s, k-1}$ for all $i \in \ccalI$
\label{line:Xi_new}
\STATE$\bbxi_{i, s, k}\gets  \sum_{j \in \mathcal{N}_{s,k}} [W_k]_{sj}\bbxi_{i, s, k-1}$ for all $i \in \ccalI$
\label{line:xi_new}
\STATE 
\agn*{
\!\!\!\ccalD_{s,k} \gets \Big( \big((n\Xi_{1,s,k})^{-1} \bbxi_{1, s, k}, \dots, (n\Xi_{m,s,k})^{-1} \bbxi_{m,s,k}\big), \\
\left( n\Xi_{1, s, k} , \dots, n\Xi_{m,s,k} \right)\Big)
}
\label{line:D}
\STATE $\bbv_{s,k}\gets \proj_{\ccalX_0} \left(  \bbp_{s,k} - \alpha_k f' \left(\bbp_{s,k}  \right)\right)$
\label{line:v}
\STATE Choose random $\o_{s,k} \in \ccalI$, compute $h_+(\bbv_{s,k}; \ccalD_{s,k}, \omega_{s,k})$ and $h'_+(\bbv_{s,k}; \ccalD_{s,k}, \omega_{s,k})$
\label{line:o}
\STATE $\beta_{s,k} \gets \frac{h_+(\bbv_{s,k}; \ccalD_{s,k}, \omega_{s,k})}{\norm{h'_+(\bbv_{s,k}; \ccalD_{s,k}, \omega_{s,k})}^2}$
\label{line:b}
\STATE $\bbz_{s,k} \gets  \proj_{\ccalX_0} \left( \bbv_{s,k} - \beta_{s,k} h'_+ \left(\bbv_{s,k}; \ccalD_{s,k}, \omega_{s,k} \right)\right)$
\label{line:z}
\ENDFOR
\RETURN Updated data $\ccalD_s(t) \gets \ccalD_{s,k}$ and control input $\bbu_s$, taken from $\bbz_{s,k}$
\end{algorithmic}
\end{algorithm}

\subsection{Distributed Information Filtering}
We solve problem~\eqref{eq:distributed} using an iterative approach.
Let $\ccalD_{s,k}$ denote the copy of the data set $\ccalD$ that is local to sensor $s$ at iteration $k \in \mbN$, where $\mbN = \set{1,2,3,\dots}$ denotes the natural numbers.
Additionally, let $\bbz_{s,k}$ denote the decision variables of the sensor $s$ at iteration $k$.
The first step of the algorithm is for sensor $s$ to broadcast $\bbz_{s,k-1} $ and a summary of $\ccalD_{s,k-1}.$
In Algorithm~\ref{alg:filter+opt}, the \emph{data summary for state} $i$ is captured by the intermediate information matrix $\bbXi_{i, s, k}$ and information vector $\bbxi_{i, s, k},$ defined in lines \ref{line:Xi} and \ref{line:xi}.
We refer the reader to \cite{kamal12} for a detailed discussion of Information Consensus Filtering, but we note briefly here that if the sensors each compute $\bbXi_{i,s,0}$ and $\bbxi_{i,s,0}$, and carry out the weighted averaging that is standard in all consensus algorithms, given in lines \ref{line:Xi_new} and \ref{line:xi_new}, then $n \bbXi_{i, s, k}$ and $\bbxi_{i, s, k}$ converge geometrically to the consensus posterior information.
In other words, under the ICF, we have that, for all $s \in \ccalS,$ 
$
\set{ \left( n\Xi_{1, s, k}, \dots, n\Xi_{m,s,k} \right)}_{k \in \mbN}\to S
$
for all
$i \in \ccalI$
and
$
\set{\big((n\Xi_{1,s,k})^{-1} \bbxi_{1, s, k}, \dots, (n\Xi_{m,s,k})^{-1} \bbxi_{m,s,k}\big)}_{k \in \mbN} \to \hbbx.
$
In particular, let $\ccalG_k = (\ccalS, \ccalE_k)$ represent a communication graph such that $(s,v) \in \ccalE_k$ if and only if sensors $s$ and $v$ can communicate at iteration $k$.
This defines the \emph{neighbor set},
$
\ccalN_{s,k} \triangleq \set{s} \cup \set{
v \in \ccalS \mid (v,s) \in \ccalE_k
}.
$
Sensor $s$ receives $\set{\bbz_{v,k-1}, \set{\bbxi_{i,v,k},\bbXi_{i,v,k}}_{i \in \ccalI}  \mid v \in \ccalN_{s,k}}.$
With this new information, the sensor updates its own optimization variable and data summary via using weighted averaging.
In particular, let $W_k$ be an $n \times n$ row stochastic matrix such that
$
[W_k]_{sv} > 0 \iff v \in \ccalN_s.
$
The exact assumptions for $W_k$ will be introduced later, when they are necessary for the proof of our main result.
The consensus steps are depicted in lines~\ref{line:p}, \ref{line:Xi_new}, and \ref{line:xi_new} of Algorithm~\ref{alg:filter+opt}.
Then, the agent can compute $\ccalD_{s,k}$ as shown in line~\ref{line:D}, concluding the IF portion of iteration $k.$
Note that, given initial conditions $\set{ \ccalD_{s, 0}}_{s \in \ccalS},$ the data for any sensor $s$ at time $k$ is bounded by the initial conditions as
\begin{equation}\label{eq:bound_data}
\norm{ \ccalD_{s,k}  } \leq
\max_{v \in \ccalS}
\norm{
\left(  \ccalD_{v,0}
\right)
}
,
\end{equation}
where $\norm{\cdot}$ is any norm in the space of data $\reals^{pm} \times \left( \Sym_{++} (p, \reals) \right)^m.$
This implies that the set of possible data is bounded with respect to any norm in the vector space $\reals^{pm} \times\left( \Sym (p, \reals) \right)^m.$
Note that, by the same arguments as can be found in \cite{kamal12}, under any norm, all sensors' data converge to the same consensus value, i.e.,
$
\norm{ \ccalD_{s,k} - \ccalD_{v,k } }\to 0
$
and
$
\norm{ \ccalD_{s,k} - \ccalD }\to 0
$
for all $s , v \in \ccalS.$

\subsection{Random Approximate Projections}
The optimization part of the algorithm begins in line \ref{line:v} by computing a negative subgradient of the $f$ at the local iterate, which we denote by $- f'(\bbz_{s,k})$, taking a gradient descent step, and projecting the iterate back to the simple constraint set $\ccalX_0$.
The positive step size $\alpha_k$ is diminishing, non-summable and square-summable, i.e, $\alpha_k \to 0$, $\sum_{k \in \mbN}\alpha_k = \infty$ and $\sum_{k \in \mbN}\alpha_k^2 < \infty$.
The last step of our method is not a projection to
the feasible set, but instead a subgradient step in the direction
that minimizes the violation of a \textit{single} LMI constraint.
Essentially, each node selects one of
the LMI’s randomly, measures the violation of that constraint,
and takes an additional gradient descent step with step size
$\beta_{s,k}$ in order to minimize this violation.
We call this process \textit{random approximate projection}.

Specifically, line~\ref{line:o} randomly assigns $\o_{s,k} \in \ccalI$, effectively choosing to activate $h\left(\cdot; \cdot, \omega_{s,k} \right)$.
The idea behind line~\ref{line:z} is to minimize a scalar metric that measures the amount of violation of $h\left(\cdot; \cdot, \omega_{s,k} \right)$, which we denote by $h_+\left(\cdot; \cdot, \omega_{s,k} \right),$ while remaining feasible.
Let $\proj_+ \colon \Sym(p,\reals) \to \Sym_+(p,\reals)$ denote the projection operator onto the set of positive semidefinite matrices, given by
\begin{equation}\label{eq:pi}
\proj_+X = E  \diag\left[ (\lambda_1)_+, \dots, (\lambda_p)_+\right] E^\top,
\end{equation}
where
$(\cdot)_+ \triangleq \max \set{\cdot, 0}$
and $\lambda_1, \dots,  \lambda_p$ are the eigenvalues and $E$ is a matrix with columns $\bbe_1, \dots, \bbe_p$ that are the corresponding eigenvectors of $X$.
A measure of the ``positive definiteness'' of $X$ is thus
$
\norm{\proj_+  X  }_F
\equiv  \sqrt{ (\lambda_1)_+^2 +  \cdots + (\lambda_p)_+^2}.
$
This defines our measure of constraint violation,
\begin{align}
 h_+ \left(\bbz; \ccalD, i\right) &\triangleq \norm{ \proj_+ h \left(\bbz; \ccalD, i\right)   }_F.
\end{align}
We use special notation to denote the vector of derivatives of $h_+$ with respect to the decision variables $\bbz$ evaluated at $(\bbz; \ccalD, i)$.
In particular, define the gradient
$
h'_+ \left(\bbz; \ccalD, i\right)= \left(\nabla_1 h_+ (\bbz; \ccalD,i), \dots  ,  \nabla_{ m+qn} h_+ (\bbz ; \ccalD,i) \right)^\top 
= \nabla_\bbz h_+(\bbz; \ccalD, i).
$
Defining the function 
$
\psi_{ij}  \colon \Sym (p, \reals) \times \reals^{m +qn} \times \left(\reals^{pm} \times  \left( \Sym_{+} (p, \reals) \right)^m \right) \to \reals
$
for notational convenience, the $j$-th entry of $h'_+  \! \left(\bbz; \ccalD, i\right)$ is given by
\begin{equation}\label{eq:dh+_dzj}
\nabla_{j} h_+ (\bbz; \ccalD,i)
=  \psi_{ij} \Big( \proj_+  \big(  h\left(\bbz;\ccalD, i \right) \big), \bbz, \ccalD\Big),
\end{equation}
where
\begin{equation}\label{eq:psi}
\psi_{ij} (X,\bbz,\ccalD)\triangleq\begin{cases}
\frac{
{{\rm Tr}} 
\left(
\nabla_j h (\bbz, \ccalD,i)
X 
\right) }{\norm{X}_F}
&{{\rm if}} \; \norm{X}_F\! >0\\
d&{{\rm else}}
\end{cases}
\end{equation}
and $d >0$ is some arbitrary constant.
The function $\psi_{ij}$ is not continuous at $X=\bb0,$ and we set $d \neq 0$ for technical reasons discussed below.
The step size
\begin{equation}\label{eqn:beta}
\beta_{s,k} = \frac{h_+(\bbv_{s,k}; \ccalD_{s,k}, \omega_{s,k})}{\norm{h'_+(\bbv_{s,k}; \ccalD_{s,k}, \omega_{s,k})}^2}
\end{equation}
given in line \ref{line:b} is a variant of the Polyak step size \cite{polyak-sg}.
Note that this is a well defined step size as we let 
$
h_+' (\bbv_{s,k}; \ccalD_{s,k}, \omega_{s,k}) = d {\bf 1}_{m+qn}
$
(cf. Eq \eqref{eq:psi})
whenever 
$
h_+(\bbv_{s,k}; \ccalD_{s,k}, \omega_{s,k})=0
$
and
it is nonzero elsewhere by construction.

%% file: reviews-response/sep.tex
the global data in $\ccalD$ are not known to the sensors and need to be computed using the ICF concurrently with the optimization of the decision variables $\bbz_s$. The ICF is an iterative process by which every sensor updates a local copy of the global data until all sensors agree on the true values in $\ccalD$. Therefore, before the ICF has converged, disagreement errors on the local copies of the  data set $\ccalD$ exist, which means that the gradients and objective function evaluations at every sensor node will not depend on the same data. In other words, gradients and objective function evaluations will be based on inexact data. An additional challenge is that the number of LMI constraints can grow very large with the number of sensors and hidden states.  We address these challenges in the remainder of this section.

%% file: text-16p/assume2.tex
In this section, we discuss our assumptions and state the main result. The detailed proofs are deferred to Section \ref{sec:converge}.
The feasible set is
\(
\ccalX = \ccalX_0 \cap \{\bbz \in \reals^{m +qn}\mid h\left(\bbz; \ccalD, i \right) \preceq 0,~\forall i \in \ccalI\}.
\)
The optimal value and solution set of the problem \eqref{eq:distributed} is
$
f^* = \min_{\ccalX} f(\bbz)$ and $ \ccalX^* = \{\bbz \in \ccalX \mid  f(\bbz) = f^*\}.
$
Note that $\bb0 \in \Gamma \times \ccalU$ and $h(\bb0; \ccalD, i) \preceq 0$ for all possible $\ccalD$ and $i,$ thus there is always a feasible point.

At step $k \in \mbN$, $\ccalD_{s,k}$ in line \ref{line:D} of Algorithm \ref{alg:filter+opt} may not yet have converged to $\ccalD,$ so we denote the
error in the constraint violation by
\begin{equation}\label{eq:constrainterror}
\nu_{s,k} =  h_+ \left(\bbz_{s,k}; \ccalD_{s,k}, \omega_{s,k} \right) - h_+ \left(\bbz_{s,k}; \ccalD, \omega_{s,k} \right) .
\end{equation}
Similar to the definition in \eqref{eq:constrainterror}, we define the error in the gradient of the constraint violation to be
\begin{equation}\label{eq:differror}
\bbdelta_{s,k}=
h'_+\left(\bbz_{s,k}; \ccalD_{s,k}, \omega_{s,k} \right)-h'_+\left(\bbz_{s,k}; \ccalD, \omega_{s,k} \right).
\end{equation}
The vectors $\nu_{s,k}$ and $\bd_{s,k}$ will enable us to quantify the joint effect of the function value error  $\nu_{s,k}$ and the subgradient error $\bd_{s,k}$
in studying convergence properties of Algorithm~\ref{alg:filter+opt}.

\subsection{Bounds}
The norms of $\nu_{s,k}$ and $\bd_{s,k}$ are bounded, i.e., for some scalars $D, N > 0$, there holds for all $k \in \mbN$ and $s \in \ccalS$  \textit{a.s.}:
$
|\nu_{s,k}|\le N$ and $\|\bd_{s,k}\| \le D.
$
We calculate the exact values of $D$ and $N$ in Corollaries~\ref{cor:N} and \ref{cor:D}, respectively.

The set $\ccalX_0 = \Gamma \times \ccalU$ is convex and compact.
Therefore, there must exist a constant $C_{\zb} > 0$ such that for any $\zb_1, \zb_2 \in \Xc_0$
\begin{equation}\label{eqn:C_z}
\|\zb_1 -\zb_2\| \le C_{\zb}.
\end{equation}
Also, note that the functions $f(\cdot)$ and $h(\zb;\ccalD,i)$ for $i \in \ccalI$ are convex (not necessarily differentiable) over some open set that contains $\Xc_0$.
A direct consequence of this and the compactness of $\Xc_0$ is that the subdifferentials $\partial f (\zb)$
and $\partial h_+ (\zb;\ccalD,i)$ are nonempty over $\Xc_0$,
where note that here we use the symbol $\partial$ to refer to the subdifferential set and not a partial derivative. 
It also implies that
the subgradients $f'(\zb) \in \partial f (\zb)$ and $h_+'(\zb;\ccalD,i) \in \partial h_+ (\zb;\ccalD,i)$ are uniformly bounded over the set $\Xc_0$. That is, there is a scalar $L_f$ such that for all $f'(\zb) \in \partial f (\zb)$ and  $\zb\in \Xc_0$,
\begin{subequations}
    \begin{equation}\label{eqn:L_f1}
    \|f'(\zb)\| \le L_f,
    \end{equation}
and for any $\zb_1, \zb_2 \in \Xc_0$,
\begin{equation}\label{eqn:L_f2}
|f(\zb_1)-f(\zb_2)| \le L_f \|\zb_1-\zb_2\|.
\end{equation}
\end{subequations}
Note that for the objective function $f$ in \eqref{eq:distributed}, we have $L_f = 1.$
Also, there is a scalar $L_h$ such that for all $h_+'(\zb;\ccalD,i)\in \partial h_+ (\zb;\ccalD,i)$, $\zb \in \Xc_0$, $i \in \ccalI$,
    \begin{equation}\label{e:Lh}
    \|h_+'(\zb;\ccalD,i)\| \le L_h.
    \end{equation}
We calculate the exact value of $L_h$ in Corollary~\ref{cor:D}.
From relation \eqref{eq:differror} we have for any $i\in \ccalI$, $s \in \ccalS$, $k \in \mbN$, and $\zb\in \Xc_0$
\begin{subequations}
\begin{equation}\label{eqn:L_h1}
\|h_+'(\zb;\ccalD_{s,k},i)\| = \|h_+'(\zb;\ccalD,i) + \bbdelta_{s,k}\|\le L_h +D,
\end{equation}
and for any $\zb_1, \zb_2 \in \Xc_0$
\begin{equation}\label{eqn:L_h2}
|h_+(\zb_1;\ccalD_{s,k},i) \! - \! h_+(\zb_2;\ccalD_{s,k},i)| \le (L_h+D) \|\zb_1 \! - \! \zb_2\|.
\end{equation}
\end{subequations}

The compactness of $\Xc_0$, the boundedness of the data sequences $\ccalD_{s,k}$, and the continuity of the functions $h_+(\zb;\ccalD,i)$ for $i \in \ccalI$ imply that there exist constants $C_h >0$ such that for any $s \in \ccalS$, $k \in \mbN$, $i \in \ccalI$ and $\zb \in \Xc_0$
\begin{subequations}
\begin{equation}\label{eqn:C_h1}
|h_+(\zb;\ccalD,i)| \le C_h,
\end{equation}
\begin{equation}\label{eqn:C_h2}
|h_+(\zb;\ccalD_{s,k},i)|  = |h_+(\zb;\ccalD,i)+\nu_{s,k}| \le C_h+N,
\end{equation}
\end{subequations}
where the second relation follows from \eqref{eq:constrainterror}.
Furthermore, when $h_+(\zb;\ccalD_{s,k},i) \neq 0$, we have $h_+'(\zb;\ccalD_{s,k},i) \neq \bb0$. Therefore, there must exist a constant $c_h >0$ such that 
\begin{align}\label{eqn:dhbnd}
\|h_+'(\zb;\ccalD_{s,k},i)\| \ge c_h
\end{align}
for all $\zb \in \Xc_0$, $s \in \ccalS$, $k \in \mbN$, and $i \in \ccalI$.
This and relation (\ref{eqn:C_h2}) imply that
\begin{align}\label{eqn:betabnd}
\b_{s,k} = \frac{h_+(\zb;\ccalD_{s,k},\o_{s,k})}{\|h'_+(\zb;\ccalD_{s,k},\o_{s,k})\|^2} \le \frac{C_h+N}{c_h^2}.
\end{align}
Note that when $h_+(\zb;\ccalD_{s,k},\o_{s,k}) = 0$, the above bound holds trivially.

\subsection{Assumptions}
In the preceding sections, we have made extensive use of the function $Q\colon \reals^q \times \reals^p \to \Sym_{+}(p,\reals).$  
For our algorithm to converge, we require the following to hold true regarding this information model.

\begin{assumption}
\label{assume:Q}
We assume that the information function $Q$
\begin{enumerate}[(a)]
\item is bounded,\label{a:bd}
\item is twice subdifferentiable\label{a:sub}
\item has bounded subdifferentials up to the second order, and\label{a:bd_sub}
\item has relatively few critical points, i.e., the sets of critical points of $Q$ and its partial derivatives up to the second order are measure zero. \label{a:crit}
\end{enumerate}
Denote the bound on the magnitude of $Q$ by $\eta_0$, the bound on its first derivatives by $\eta_1$, and the bound on its second derivatives by $\eta_2.$
\end{assumption}
The technical restrictions on the information model $Q$ are not overly restrictive in practice.
In particular, items~\ref{a:bd}, \ref{a:sub}, and \ref{a:bd_sub} imply that one cannot obtain infinite information and that, by changing sensor configuration, the information rate (and the rate of change of the rate) cannot change infinitely quickly.
Item~\ref{a:crit} essentially allows us to distinguish between sensor configurations, i.e., the set of configurations that offer ``optimal'' information rates are relatively sparse with respect to the space of signal source positions and sensor configurations $\reals^q \times \reals^p.$

The next two assumptions are related to the random variables $\o_{s,k}$.
At each iteration $k$ of the inner-loop, recall that each sensor $s$ randomly generates $\o_{s,k} \in \ccalI$. We assume that they are \textit{i.i.d.} samples from some probability distribution on $\ccalI$.
\begin{assumption}\label{assume:Pr}
In the $k$-th iteration of the inner-loop,
each element $i$ of $\ccalI$ is generated with nonzero probability,
i.e., for any $s \in \ccalS$ and $k \in \mbN$ it holds that
$
\pi_i \triangleq \textrm{Pr}\{ \omega_{s,k} = i\} > 0.
$
\end{assumption}
Assumption~\ref{assume:Pr} is easy to satisfy because $\ccalI$ is a finite set.
\begin{assumption}\label{assume:c}
For all $s \in \ccalS$ and $k \in \mbN$,
there exists a constant $\k> 0$ such that for all $\zb \in \Xc_0$
$
 \dist^2 (\zb,\Xc)  \le \k \Es\left[h_+^2(\zb;\ccalD_{s,k},\omega_{s,k})\right].
$
\end{assumption}
\noindent The upper bound in Assumption \ref{assume:c} is known as \textit{global error bound} and is crucial for the convergence analysis of Algorithm \ref{alg:filter+opt}.
Sufficient conditions for this bound have been shown in \cite{pang-book,Lewis96errorbounds,Bauschke:1996,Gubin19671}, which includes the case when each function $h(\cdot;\cdot,i)$ is linear in $\bbz$,
or when the feasible set $\Xc$ has a nonempty interior.

\begin{assumption}\label{assume:network}
For all $k \in \mbN$, the weighted graphs $\Gc_k = (\ccalS,\Ec_k,W_k)$ satisfy:
\begin{enumerate}[(a)]
\item There exists a scalar $\eta \in (0,1)$ such that $[W_k]_{sj} \ge \eta$ if $j \in \Nc_{s,k}$. Otherwise, $[W_k]_{sj} = 0$.
\item The weight matrix $W_k$ is doubly stochastic, i.e., $\sum_{s\in\ccalS}[W_k]_{sj}=1$ for all $j \in \ccalS$ and $\sum_{j\in\ccalS}[W_k]_{sj}=1$ for all $s \in \ccalS$.
\item There exists a scalar $Q > 0$ such that the graph $\left(\ccalS,\cup_{\ell=0,\ldots,Q-1} \Ec_{t+\ell}\right)$ is strongly connected for any $t \ge 1$.
\end{enumerate}
\end{assumption}
\noindent This assumption ensures a balanced communication between sensors. It also ensures that there exists a path from one sensor to every other sensor infinitely often even if the underlying graph topology is time-varying.

\subsection{Main result}
Our main result shows the almost sure convergence of Algorithm~\ref{alg:filter+opt}.
Specifically, the result states that the sensors asymptotically reach an agreement to a random point $\bbz^*$ which is in the optimal set $\ccalX^*$ almost surely (a.s.), as given in the following theorem.
\begin{theorem}[Asymptotic Convergence Under Noise]
\label{prop:main}
Let Assumptions
\ref{assume:Q} - \ref{assume:network}
hold.
Let also the stepsizes be such that $\sum_{k=1}^{\infty} \alpha_k = \infty$ and $\sum_{k=1}^{\infty} \alpha_k^2 < \infty$.
Then, the iterates $\{\bbz_{s,k}\}$ generated by each agent $s \in \ccalS$ via Algorithm~\ref{alg:filter+opt}
converge almost surely to the same point in the optimal set $\ccalX^*$, i.e., for a random point $\bbz^*\in \ccalX^*$
$
\lim_{k\to\infty} \bbz_{s,k} = \bbz^*, \quad \forall s \in \ccalS \quad a.s.
$
\end{theorem}

%% file: text-16p/converge2.tex
\subsection{Preliminary Results}
\begin{lemma}[Non-expansiveness \cite{BNO}]
\label{lem:proj}
Let $\mathcal{X} \subseteq \mathbb{R}^d$ be a nonempty closed convex set.
The function $\mathsf{\Pi}_{\mathcal{X}}: \mathbb{R}^d \rightarrow \mathcal{X}$ is nonexpansive, i.e.,
$
\|\mathsf{\Pi}_{\mathcal{X}}[\bba]-\mathsf{\Pi}_{\mathcal{X}}[\bbb]\| \leq \|\bba - \bbb\|
$
for all
$
\bba , \bbb \in \mathbb{R}^d.
$
\end{lemma}

\begin{lemma}[\!\!{\cite[Lemma 10-11, p. 49-50]{polyak}}]
\label{thm:super}
Let 
$
\Fc_k \triangleq \set{v_\ell, u_\ell, a_\ell, b_\ell}_{\ell=0}^k
$
denote a collection of nonnegative real random variables for $k \in \mbN \cup \set{\infty}$ such that
$
\mathsf{E}[v_{k+1}|\Fc_k] \leq (1+a_k)v_k - u_k + b_k
$
for all
$
k \in \set{0}\cup \mbN \quad a.s.
$
Assume further that 
$
\set{a_k}
$
and
$
\set{b_k}
$
are \textit{a.s.} summable.
Then, we have \textit{a.s.} that (i)
$
\set{u_k}
$
is summable and (ii) there exists a nonnegative random variable
$
v
$
such that
$
\set{v_k} \to v.
$
\end{lemma}

In the next lemma, we relate the two iterates $\pb_{s,k}$ and $\zb_{s,k}$ in Line \ref{line:p} and \ref{line:z} of Algorithm \ref{alg:filter+opt}.
In particular, we show a relation of $\pb_{s,k}$ and $\zb_{s,k-1}$
associated with any convex function $g$ 
which will be often used in the analysis. For example, $g(\zb) = \|\zb-\ab\|^2$ for some $\ab\in\mathbb{R}^n$ or
$g(\zb) = \dist^2 (\zb,\Xc)$.
\begin{lemma}[Convexity and Double Stochasticity]
\label{lem:ds}
For any convex function $g : \mathbb{R}^n \to \mathbb{R}$, we have
\[
\sum\nolimits_{s\in \ccalS} g(\pb_{s,k}) \le  \sum\nolimits_{s\in \ccalS} g(\zb_{s,k-1})
\]
\end{lemma}
\begin{proof}
The double stochasticity of the weights plays a crucial role in this lemma. From the definition of $\pb_{s,k}$ in Line \ref{line:p} of Algorithm \ref{alg:filter+opt}, we obtain
\begin{align*}
\sum_{s\in \ccalS} g(\pb_{s,k}) \le~& \sum_{s\in \ccalS} \sum_{j\in \ccalS} [W_k]_{sj} g(\zb_{j,k-1})\\
=~& \sum_{j\in \ccalS} \left(\sum_{s\in \ccalS} [W_k]_{sj}\right)g(\zb_{j,k-1})= \sum_{j\in \ccalS} g(\zb_{j,k-1}).
\end{align*}
\end{proof}

Lastly, for the convergence proof of our algorithm, we use a result from \cite{Ram2012}, which ensures successful consensus in the presence of a well behaved disturbance sequence.
\begin{lemma}[Perturbed Consensus]
\label{lemma:ram}
Let Assumption \ref{assume:network} hold.
Consider the iterates generated by
\begin{equation}\label{eqn:rule}
\tb_{s,k} \hspace{-0.5mm}= \hspace{-0.5mm}\sum_{v\in\ccalS} [W_k]_{sv} \tb_{v,k-1} + \eb_{s,k},~ \forall s \in \ccalS,
\end{equation}
Suppose there exists a nonnegative nonincreasing scalar sequence $\{\alpha_k\}$ such that
$
\sum_{k=1}^{\infty} \alpha_k \|\eb_{s,k}\|< \infty \text{ for all } s \in \ccalS.
$
Then, for all $s,v \in \ccalS$,
$
\sum_{k=1}^{\infty} \alpha_k \|\tb_{s,k} -\tb_{v,k}\|  < \infty.
$
\end{lemma}

In addition to the well-known results of Lemmas~\ref{lem:proj}-\ref{lemma:ram}, we need the following intermediate results before the main result in Theorem~\ref{prop:main} can be proven.
The proofs of all of these Lemmas are deferred to Appendices~\ref{app:summability} and \ref{app:lemmaproofs2}.
In the first two of them, we posit that the two error sequences $\{\nu_{s,k}\}$
and $\{\bbdelta_{s,k}\}$ are summable.
\begin{lemma}[Summable Constraint Errors]
\label{lem:summable_constraints}
For almost every bounded sequence $\set{\bbz_{s,k}},$
the error in the constraint $\{|\nu_{s,k}|\}_{k \in \mbN}$ is summable a.s., i.e.,
$
\sum_{k=0}^\infty \abs{\nu_{s,k}  }<\infty.
$
\end{lemma}
\begin{lemma}[Summable Constraint Violation Errors]
\label{lem:summable_grads}
For almost every bounded sequence $\set{\bbz_{s,k}},$
the error in the gradient of the constraint $\{\|\bbdelta_{s,k}\|\}_{k \in \mbN}$ is summable a.s., i.e.,
$
\sum_{k=0}^\infty \norm{\bbdelta_{s,k} }<\infty.
$
\end{lemma}
In the next lemma, we relate the two iterates $\pb_{s,k}$ and $\zb_{s,k}$ in Line \ref{line:p} and \ref{line:z} of Algorithm \ref{alg:filter+opt}.
\begin{lemma}[Basic Iterate Relation]
\label{lem:basic}
Consider the sequences $\{\zb_{s,k}\}_{k =0}^\infty$ and $\{\pb_{s,k}\}_{k =0}^\infty$ for $s \in \ccalS$ generated by Algorithm \ref{alg:filter+opt}. Then, for any $\zb,\check{\zb} \in \Xc$, $s \in \ccalS$ and $k \in \mbN$, we have \textit{a.s.}:
\begin{align*}
&\|\zb_{s,k}-\zb\|^2
\le  \|\pb_{s,k}-\zb\|^2 - 2\a_k(f(\check{\zb})-f(\zb))\\
&~+ 2\frac{C_{\zb}(C_h+N)}{c_h^2}\|\bd_{s,k}\| +2\frac{C_h+N}{c_h^2}|\nu_{s,k}|\\
&~ + \frac{1}{4\eta}\|\pb_{s,k} - \check{\zb}\|^2 -\frac{\tau-1}{\tau(L_h+D)^2}h_+^2(\pb_{s,k};\ccalD_{s,k},\o_{s,k})\\
&~ + A_{\eta,\tau}\a_k^2L_f^2,
\end{align*}
where $A_{\eta,\tau} = 1+4\eta+\tau$ and $\eta,\tau>0$ are arbitrary.
\end{lemma}
Since we use a random approximate projection, we cannot guarantee the feasibility of the sequences $\pb_{s,k}$ and $\zb_{s,k}$ for every $k \in \mbN$ and $s \in \ccalS$. In the next lemma, we prove that $\pb_{s,k}$ for all $s \in \ccalS$ asymptotically achieve feasibility even under the effect of the disturbances $\{\nu_{s,k}\}$
and $\{\bbdelta_{s,k}\}$.
\begin{lemma}[Asymptotic Feasibility Under Noise]
\label{lem:feas}
Let Assumption
\ref{assume:Pr} and
\ref{assume:c} hold. 
Let $\a_k$ be square summable Consider the sequences $\{\pb_{s,k}\}$ for $s \in \ccalS$ generated by Algorithm \ref{alg:filter+opt}.
Then, for any $s \in \ccalS$, the sequence
$
\set{ \dist(\pb_{s,k},\Xc) }_{k \in \mbN}
$
is \textit{a.s.} square summable.
\end{lemma}
In the final Lemma, we show that $\|\zb_{s,k}-\pb_{s,k}\|$ eventually converges to zero for all $s \in \ccalS$.
This result combined with Lemma \ref{lem:feas} implies that the sequences $\set{ \zb_{s,k} }_{k \in \mbN}  $ also achieve asymptotic feasibility.
\begin{lemma}[Network Disagreement Under Noise]
\label{lem:disagree}
Let Assumptions
\ref{assume:Pr}-\ref{assume:network} hold. Let the sequence
$\{\a_k\}$ is such that $\sum_{k=1}^\infty \a_k^2<\infty$.
Define
$\eb_{s,k}=\zb_{s,k}-\pb_{s,k}$ for all $s\in \ccalS$ and $k\ge 1$.
Then, we have \textit{a.s.}:
\begin{enumerate}[(a)]
\item
$\sum_{k=1}^\infty \|\eb_{s,k}\|^2 <\infty \quad \text{for all } s \in \ccalS$,
\item
$\sum_{k=1}^\infty\a_k\|\check{\zb}_{s,k}-\bar {\check{\zb}}_k\|<\infty    \quad \hbox{for all $s \in \ccalS$}$,
\end{enumerate}
where $\check{\zb}_{s,k} = \mathsf{\Pi}_{\Xc}[\pb_{s,k}]$ and $\bar {\check{\zb}}_k = \frac{1}{n}\sum_{s \in \ccalS}\check{\zb}_{s,k}$.
\end{lemma}

Note that the sequences $\{\zb_{s,k}\}_{k \in \mbN}$ for $s \in \ccalS$ generated by Algorithm \ref{alg:filter+opt} can be
represented as relation \eqref{eqn:rule}. That is,
$
\zb_{s,k}= \sum_{s\in\ccalS} [W_k]_{sj} \zb_{j,k-1} + \eb_{s,k}
$
for all $s \in \ccalS$
and $\eb_{s,k} = \zb_{s,k}-\pb_{s,k}$.
Therefore, from Lemma \ref{lem:disagree}(a), we have
$
 \sum_{k=1}^\infty \a_k\|\eb_{s,k}\| =  \frac{1}{2}\sum_{k=1}^\infty \a_k^2 + \frac{1}{2}\sum_{k=1}^\infty \|\eb_{s,k}\|^2<\infty.
$
Invoking Lemma \ref{lemma:ram}, we see that there is consensus among $\zb_{s,k}$ for $s \in \ccalS$.

\subsection{Proof of Theorem~\ref{prop:main}}
We invoke Lemma \ref{lem:basic} with $\check{\zb} = \check{\zb}_{s,k}$ (Note that $\check{\zb}_{s,k}$ is defined as $\mathsf{\Pi}_{\Xc}[\pb_{s,k}]$ in Lemma \ref{lem:disagree}), $\tau = 4$ and $\eta = \k(L_h+D)^2$.
We also let
$\zb = \zb^*$ for an arbitrary $\zb^* \in \Xc^*$. Therefore, for any $\zb^* \in \Xc^*$, $s \in \ccalS$ and $k \in \mbN$, we have \textit{a.s.}:
\begin{align}\label{eqn:pr1}
&\|\zb_{s,k}-\zb^*\|^2
\le  \|\pb_{s,k}-\zb^*\|^2 - 2\a_k(f(\check{\zb}_{s,k})-f(\zb^*))\nonumber\\
&~+ 2\frac{C_{\zb}(C_h+N)}{c_h^2}\|\bd_{s,k}\| +2\frac{C_h+N}{c_h^2}|\nu_{s,k}|\nonumber\\
&~ + \frac{1}{4\k(L_h+D)^2}\dist^2(\pb_{s,k},\Xc)+A\a_k^2L_f^2\nonumber\\
&~-\frac{3}{4(L_h+D)^2}h^2_+(\pb_{s,k};\ccalD_{s,k},\o_{s,k}),
\end{align}
where $A = 5+4\k(L_h+D)^2$. From Assumption \ref{assume:c}, we know that
$
\dist^2(\pb_{s,k},\Xc) \le \k \Es\left[h^2_+(\pb_{s,k};\ccalD_{s,k},\o_{s,k})\mid \Fc_{k-1}\right].
$

Denote by $\Fc_k$ the sigma-field induced by the history of the algorithm up to time $k$, i.e.,
$
\Fc_k = \{\zb_{s,0},~(\o_{s,\l},~ 1 \le \l \le k),~ s\in \ccalS\} \quad \text{for } k \in \mbN,
$
and $\Fc_0 = \{\zb_{s,0},~ s\in \ccalS\}$.
Taking the expectation conditioned on $\Fc_{k-1}$ in relation (\ref{eqn:pr1}), summing this over $s \in \ccalS$, and using the above relation, we obtain
\begin{align}\label{eqn:pr2}
&\sum_{s\in\ccalS}\Es\left[\|\zb_{s,k}-\zb^*\|^2\mid \Fc_{k-1}\right]\\
&~\le  \sum_{s\in\ccalS}\|\zb_{s,k-1}-\zb^*\|^2 - 2\a_k\sum_{s\in \ccalS}(f(\check{\zb}_{s,k})-f(\zb^*))\nonumber\\
&~~+ 2\frac{C_{\zb}(C_h+N)}{c_h^2}\sum_{s\in\ccalS}\Es[\|\bd_{s,k}\|\mid \Fc_{k-1}] \nonumber\\ &~~+2\frac{C_h+N}{c_h^2}\sum_{s\in\ccalS}\Es[|\nu_{s,k}|\mid \Fc_{k-1}]\nonumber\\
&~ -\frac{1}{2\k_2(L_h+D)^2}\sum_{s\in\ccalS}\dist^2(\pb_{s,k},\Xc) + A\a_k^2nL_f^2,\nonumber
\end{align}
where we used Lemma \ref{lem:ds} for the first term on the right-hand side. 
Recall that $f(\zb)=\sum_{s\in\ccalS} f(\zb)$ and $\bar{\check{\zb}}_k \triangleq \frac{1}{n}\sum_{\ell \in \ccalS} \check{\zb}_{\ell,k}$.
Using $\bar{\check{\zb}}_k$ and $f$,
we can rewrite the term $\sum_{s\in\ccalS}(f(\check{\zb}_{s,k})-f(\zb^*))$ as follows:
\begin{subequations}
\begin{align}
\sum_{s\in\ccalS} &(f(\check{\zb}_{s,k})-f(\zb^*)) \nonumber\\
 &~=  \sum_{s\in\ccalS} (f(\check{\zb}_{s,k})-f(\bar{\check{\zb}}_k)) +  (f(\bar{\check{\zb}}_k)-f^*)\label{eqn:rewrite1}\\
&~\le \sum_{s\in\ccalS} \langle f'(\bar{\check{\zb}}_k),\bar{\check{\zb}}_k-\check{\zb}_{s,k}\rangle+  (f(\bar{\check{\zb}}_k)-f^*)\label{eqn:rewrite2}\\
&~ \le \sum_{s\in\ccalS}  \|f'(\bar{\check{\zb}}_k)\|\,\|\check{\zb}_{s,k}-\bar{\check{\zb}}_k\|+  (f(\bar{\check{\zb}}_k)-f^*)\label{eqn:rewrite3}\\
&~ \le L_f\sum_{s\in\ccalS}  \|\check{\zb}_{s,k}-\bar{\check{\zb}}_k\|+  (f(\bar{\check{\zb}}_k)-f^*),\label{eqn:rewrite4}
\end{align}
\end{subequations}
where (\ref{eqn:rewrite1}) follows from adding and subtracting $f(\bar{\check{\zb}}_k)$;
(\ref{eqn:rewrite2}) follows from the convexity of the function $f$;
(\ref{eqn:rewrite3}) follows from the Schwarz inequality;
and (\ref{eqn:rewrite4}) follows from relation \eqref{eqn:L_f1} and $\bar{\check{\zb}}_k \in \Xc \subseteq \Xc_0$.
Combining (\ref{eqn:rewrite4}) with (\ref{eqn:pr2}), we obtain
\begin{align*}
&\sum_{s\in\ccalS}\Es\left[\|\zb_{s,k}-\zb^*\|^2\mid \Fc_{k-1}\right]\\
&~\le  \sum_{s\in\ccalS}\|\zb_{s,k-1}-\zb^*\|^2 - 2\a_k(f(\bar{\check{\zb}}_k)-f^*)\\
&~~+ 2\frac{C_{\zb}(C_h+N)}{c_h^2}\sum_{s\in\ccalS}\Es[\|\bd_{s,k}\|\mid \Fc_{k-1}] \\ &~~+2\frac{C_h+N}{c_h^2}\sum_{s\in\ccalS}\Es[|\nu_{s,k}|\mid \Fc_{k-1}]\\
&~~ +L_f\sum_{s\in\ccalS} \a_k \|\check{\zb}_{s,k}-\bar{\check{\zb}}_k\| + A\a_k^2nL_f^2,
\end{align*}
where we omitted the negative term.

Since $\bar{\check{\zb}}_k \in \Xc$, we have $f(\bar{\check{\zb}}_k) - f^* \ge 0$.
Thus, under the assumption $\sum_{k=0}^{\infty} \alpha_k^2 < \infty$ and Lemma \ref{lem:disagree}(b), the above relation satisfies all the conditions of Lemma \ref{thm:super}. Using this, we have the following results.

\textit{Result 1:}  For some $\bbz^* \in \ccalX^*$ and any $s \in \ccalS$, the sequence $\{\|\zb_{s,k}-\zb^*\|\}$ is
convergent \textit{a.s.}

\textit{Result 2:}
$
\sum_{k=1}^{\infty} \alpha_k(f(\bar{\check{\zb}}_k)-f^*) < \infty \quad a.s.
$

As a direct consequence of \textit{Result 1},
we know that the sequences $\{\|\pb_{s,k}-\zb^*\|\}$ and $\{\|\check{\zb}_{s,k}-\zb^*\|\}$ are also
convergent \textit{a.s.}
(This is straightforward from Line \ref{line:p} of Algorithm \ref{alg:filter+opt} and Lemma \ref{lem:feas} with the knowledge that
$\dist^2(\pb_{s,k},\Xc) = \|\pb_{s,k}-\check{\zb}_{s,k}\|^2$.)
Since  $\|\bar{\check{\zb}}_k-\zb^*\| \le \frac{1}{N} \sum_{i\in \ccalS}\|\check{\zb}_{s,k}-\zb^*\|$,
we further know that the sequence $\{\|\bar{\check{\zb}}_k-\zb^*\|\}$ is also
convergent \textit{a.s.}

As a direct consequence of \textit{Result 2}, since $\alpha_k$ is not summable,
$
\liminf_{k\rightarrow \infty} (f(\bar{\check{\zb}}_k)-f^*)=0 \; a.s.
$
From this relation and the continuity of $f$, it follows that
the sequence $\{\bar{\check{\zb}}_k\}$ must have one accumulation point in the set $\Xc^*$ \textit{a.s.}
This and the fact that $\{\|\bar{\check{\zb}}_k-\zb^*\|\}$ is convergent \textit{a.s.} for every $\zb^* \in \Xc^*$ imply that
\begin{equation}\label{eqn:z_final1}
\lim_{k \rightarrow \infty} \bar{\check{\zb}}_k = \zb^* \quad a.s.
\end{equation}
Also, by Lemma \ref{lem:disagree}(b), we have
\begin{align}\label{eqn:liminfv1}
\liminf_{k \rightarrow \infty} \|\check{\zb}_{s,k}-\bar{\check{\zb}}_k\| = 0 \quad \text{for all } s \in \ccalS \quad a.s.
\end{align}

The fact that $\{\|\check{\zb}_{s,k}-\zb^*\|\}$ is convergent \textit{a.s.} for all $s\in\ccalS$, together with~\eqref{eqn:z_final1}
and~\eqref{eqn:liminfv1}  implies that
\begin{equation}\label{eqn:consensus1}
\lim_{k \rightarrow \infty} \check{\zb}_{s,k} = \zb^* \quad \text{for all } s \in \ccalS \quad a.s.
\end{equation}
Since $\|\zb_{s,k}-\check{\zb}_{s,k}\| \le \|\zb_{s,k}-\pb_{s,k}\| + \|\pb_{s,k}-\check{\zb}_{s,k}\|$, by invoking Lemma \ref{lem:feas} and \ref{lem:disagree}(a), we have
$
\lim_{k\to \infty} \|\zb_{s,k}-\check{\zb}_{s,k}\| = 0$ for all $ s \in \ccalS \quad a.s.
$
This and relation (\ref{eqn:consensus1}) give us the desired result, which is
$
\lim_{k \rightarrow \infty} \zb_{s,k} = \zb^\star$ for all $ s \in \ccalS \quad a.s.
$




%% file: text-16p/sims2.tex
In this section, we illustrate Algorithm~\ref{alg:filter+opt} on a network of $n$ mobile sensors that collaboratively localize a collection of $m$ landmarks.
We use a distributed connectivity maintenance controller \cite{zavlanos2011graph} to ensure that the sensor network can share information and collaborate for all time.
The controller operates by signaling for a strong attractive force whenever a link is going to be deleted that can break connectivity.
The controller also deletes or replaces weak links as long as connectivity is preserved.

\subsection{Sensor Model}\label{sec:mod}
We assume that the landmarks live in $\reals^2$ and denote the configuration of landmark $i$ by $\bbx_i.$
We also assume that the configuration space of the robot is $\reals^2$.

Generic models for the measurement covariance matrix for sparse landmark localization have been proposed \cite{chung06, simonetto11, jalalkamali12}.
These models have two important characteristics that apply to a wide variety of localization sensors:
\begin{enumerate*}[label=(\itshape\roman*\upshape)]
\item Measurement quality is inversely proportional to viewing distance, and
\item The direction of maximum uncertainty is the viewing direction, i.e., the sensors are more proficient at sensing bearing than range.
\end{enumerate*}
The general idea is to use the vector $\hbbx_{i}(t)-\bbr_s(t)$ to construct a diagonal matrix in a coordinate frame relative to the sensor, then rotate the matrix to a global coordinate frame so that it can be compared to other observations.

Following the accepted generic sensor models, we express $Q(\bbr_s, \hbbx_{i})$ by its eigenvalue decomposition
$
Q= R_{\phi_{i} } \Lambda_{i}R_{\phi_{i}}^\top,
$
where the argument $(\bbr_s, \hbbx_{i})$ of $Q$, $\Lambda_i$, and $R_{\phi_i}$ is implicit.
The angle $\phi_{i} \in \left(-\frac{\pi}{2},\frac{\pi}{2} \right]$ is defined so that the first column of $R_{\phi_{i}}$, which by convention is $[\cos \phi_{i} \; \sin \phi_{i}]^\top \! \! , \;$  is parallel to the subspace spanned by the vector $\hbbx_{i}-\bbr_{s}$.
This angle is given by
$
\phi_{i}   = \tan^{-1} \left( 
					( [\hbbx_{i}]_2-[\bbr_{s}]_2)/([\hbbx_{i}]_1 - [\bbr_{s}]_1) 
				\right),
$
where the subscripts outside the brackets refer to the first and second coordinate of the vector representing the location of the landmark and sensor located at $\hbbx_{i}$ and $\bbr_s$.
The eigenvalue matrix is given by $\Lambda_{i} = \diag(\lambda_{1,i}, \lambda_{2,i}),$ where
$
\lambda_{1,i} = c_0(/1 + c_1\norm{\hbbx_{i} -\bbr_{s}}_2 ) \and \lambda_{2,i} =\rho \lambda_{1,i} .
$
In the body frame of sensor $s$, the eigenvalues $\lambda_{1,i}$ and $\lambda_{2,i} $ control the shape of the confidence ellipses for individual measurements of the $i$-th target.
The parameter $c_0 >0$ represents the overall sensor quality and scales the whole region equally, $c_1 > 0$ controls the sensitivity to depth, and $\rho >1$ controls eccentricity of the confidence regions associated with the measurements.
Figure~\ref{fig:params} illustrates the measurement model for one sensor and two targets.

\begin{figure}[t]
\begin{center}
\includegraphics[width=\columnwidth]{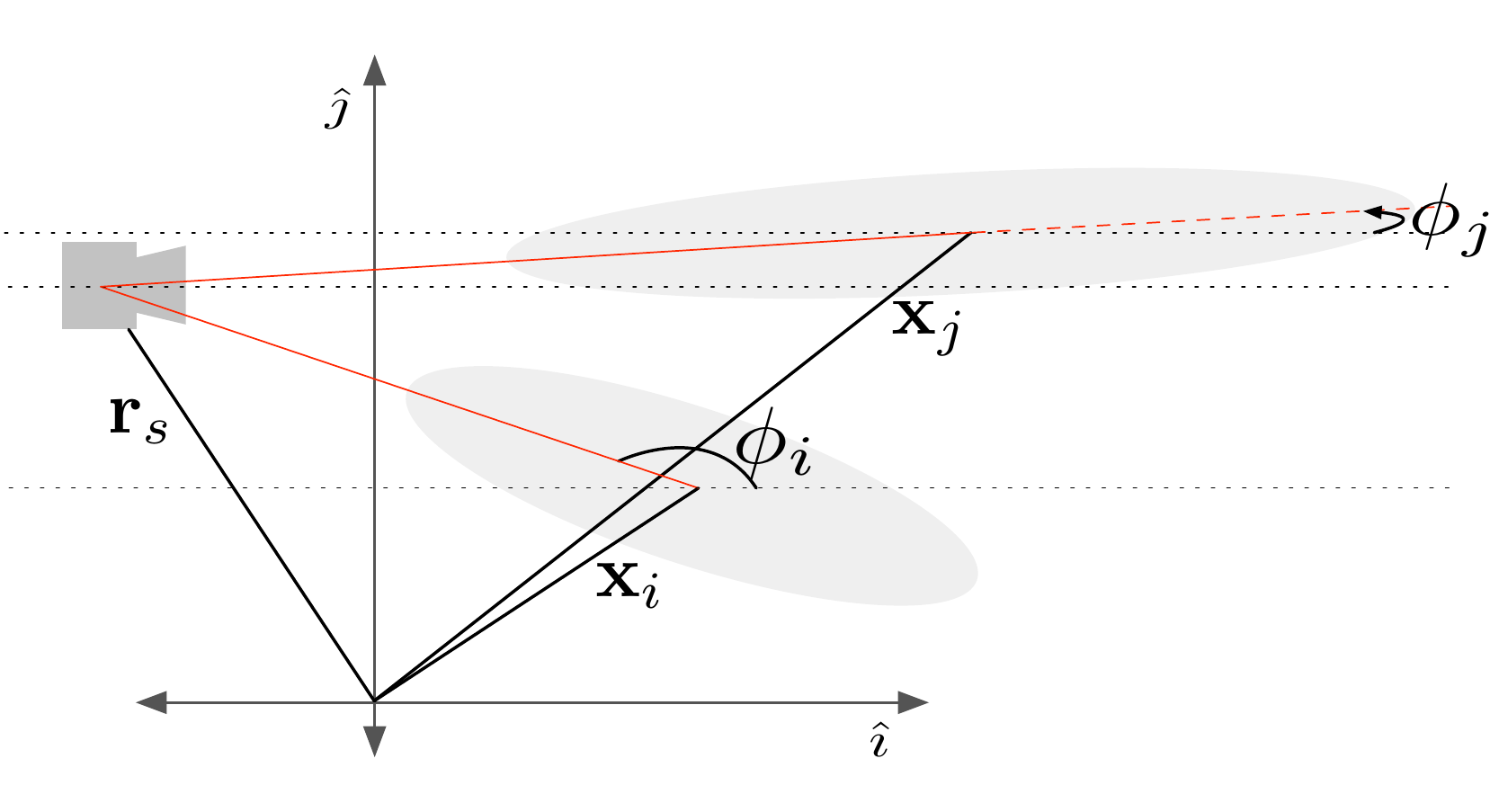}
\vspace{-1.0cm}
\caption{A sensor at $\bbr_s$ observes targets it believes to be at $\hbbx_i$ and $\hbbx_j$.
Possible confidence regions are shown in grey.
Dotted lines are parallel to the global $\hat{\imath}$ axis.
Red lines guide the eye from the sensor to targets
}
\vspace{-0.7cm}
\label{fig:params}
\end{center}
\end{figure}

Since $\tan^{-1}$ is not continuous at zero, we need to impose the following restriction so that Assumption~\ref{assume:Q} holds.
Essentially, we assume that, for the selected $\delta>0$, which recall has an affect on the robot's ability to translate in $\reals^2,$ we choose a $\tau_i$ such that 
$
\tau_i < c_0 / (1 + c_1\delta).
$
Thus, if the robot is within $\delta$ of a target $i$, the constraint $h(\bbz, \ccalD, i)$ is trivially satisfied, and therefore we will never evaluate $Q$ or any of its derivatives in this region.

\subsection{Results}
\begin{figure}[!t]
\centering
\includegraphics[width=\columnwidth]{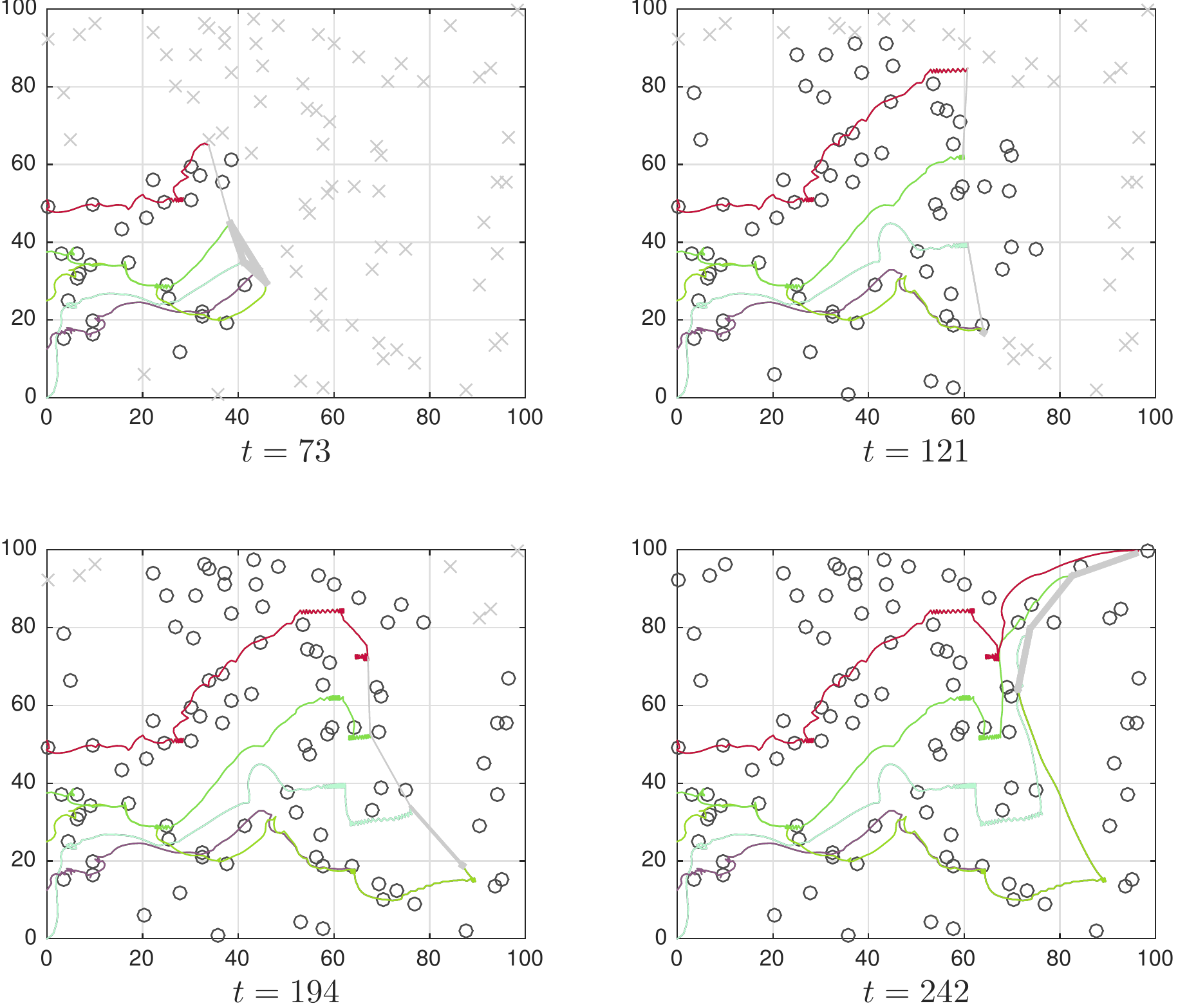}
\vspace{-0.7cm}
\caption{Four sensors localize 100 landmarks.  
Gray $\times$'s indicate landmarks with uncertainty above the threshold, and black $\circ$'s indicate localized landmarks.
Gray lines connecting the agents indicate the communication links.
}
\vspace{-0.7cm}
\label{fig:traj}
\end{figure}

Figure~\ref{fig:traj} shows an example of the resulting active sensing trajectories for a network of four sensors that cooperatively localize a uniformly random distribution of 100 landmarks in a square $100 \,{{\rm m}}^2$ workspace to a desired accuracy of $1\, {{\rm m}}$.
This corresponds to an eigenvalue tolerance of $\tau_1=\cdots=\tau_m=9\,{{\rm m}}^{-2}.$
The prior distribution on the landmarks was uninformative, i.e., the initial mean estimate and information matrix are given by the initial observations $\set{\bby_{i,s} (0)\mid i \in \ccalI, s \in \ccalS}$ and the corresponding information matrices were set to $\set{Q(\bbr_s (0), \bby_{i,s} (0))}_{  i \in \ccalI, s \in \ccalS}.$
The parameters of the sensing model $Q$ were set to $c_0 = \frac{1}{2}, c_1 = 10, \and \rho = 30$.
The feasible control set $\ccalU$ was set to
$
\ccalU = \set{\bbu_s \in \reals^2 \mid \norm{\bbu_s}\leq 1\,{{\rm m}}}.
$
In the figure, note that between $t=194$ and $t=242$, the agent in the bottom right moves north in order to maintain connectivity while the other agents move to finish the localization task.
It can be seen
from the four snapshots of the trajectories in Fig.~\ref{fig:traj}, that the robots effectively ``divide and conquer'' the task based on proximity to unlocalized targets.

\begin{figure}[!t]
\centering
\includegraphics[width=\columnwidth]{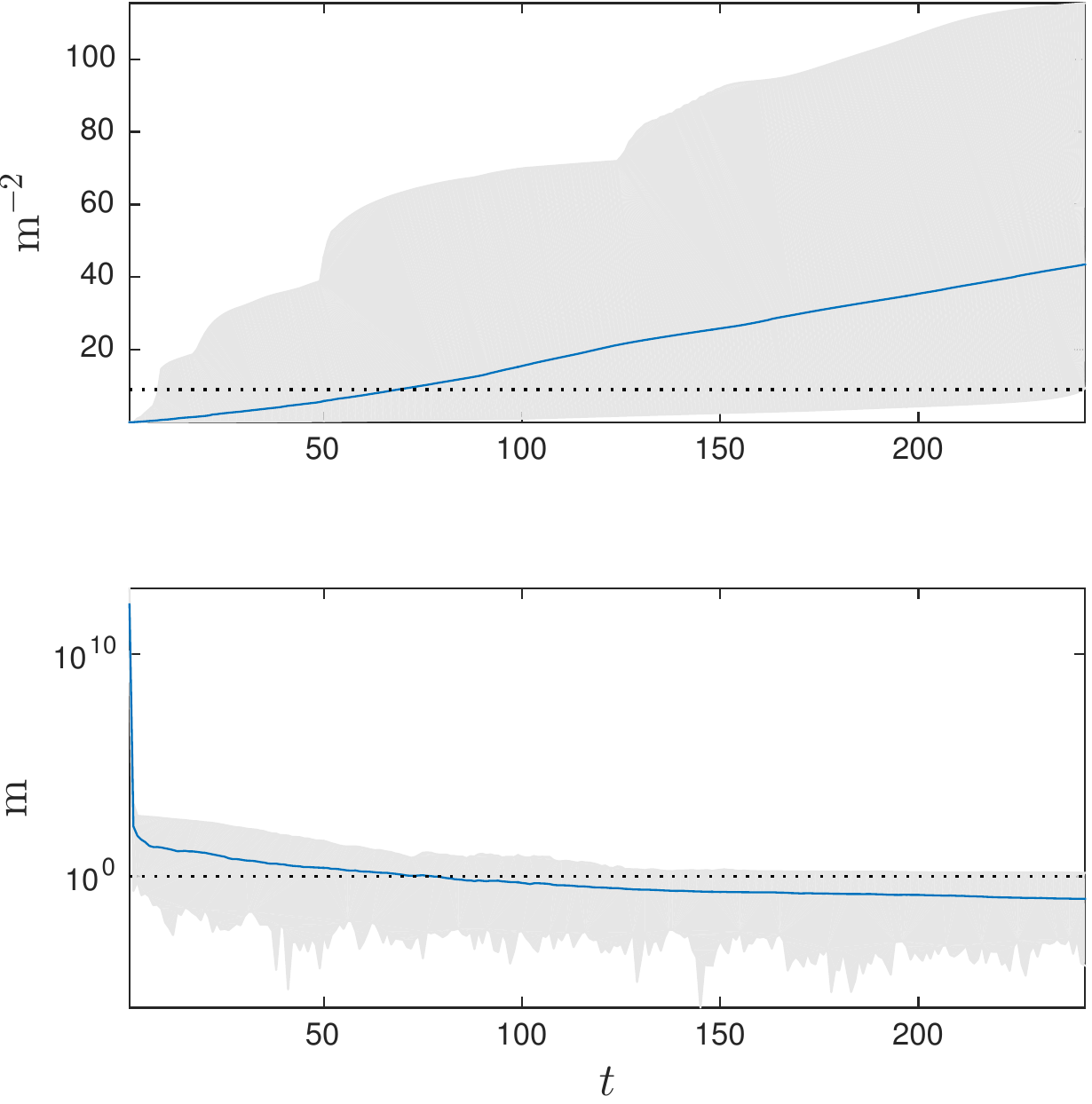}
\vspace{-0.7cm}
\caption{
Top panel:
Evolution of the minimum eigenvalues.
The center curve is the mean among all 100 targets, with the grey block extending to the lowest and largest eigenvalues at each time instant.
The dotted line is at the tolerance $\tau=9\,{{\rm m}}^{-2}$.
Bottom panel:
Evolution of localization error.
The center curve is the mean among all 100 targets, with the grey block extending to the minimum and maximum errors at each time instant.
The dotted line is at the desired accuracy of $1\,{{\rm m}}$.
Note that the scale on vertical axis is large because the initial uncertainty is infinite.
}
\vspace{-0.3cm}
\label{fig:eigs_and_errs}
\end{figure}

Fig.~\ref{fig:eigs_and_errs} shows the evolution of the minimum eigenvalues of the information matrices $S_i(t)$ from the simulation shown in Fig.~\ref{fig:traj}.
Note that, even after the minimum eigenvalue associated with a landmark has surpassed $\tau$, more information is still collected, but it does not affect the control signal.
Note also that the maximal uncertainty in the bottom panel does not go below the desired threshold, even though the minimal eigenvalue does go above the threshold in the top panel.
This is expected because we use the 95\% confidence interval to generate the threshold $\tau$, thus we expect 5\% of landmarks to have errors larger than the desired threshold.

In this case study, we selected the termination criterion for Algorithm~\ref{alg:filter+opt} to be at most 30 iterations.
When testing the algorithm with $\alpha_k = 1/k$, we observed that the relative change in the iterates $\bbz_{s,k}$ and the data $\ccalD_{s,k}$ would often reach levels at or below 10\% inside of thirty iterations, so we allowed the message passing to terminate early if this criterion was met.

%


\begin{figure}[!t]
\centering
\includegraphics[width=0.9\columnwidth]{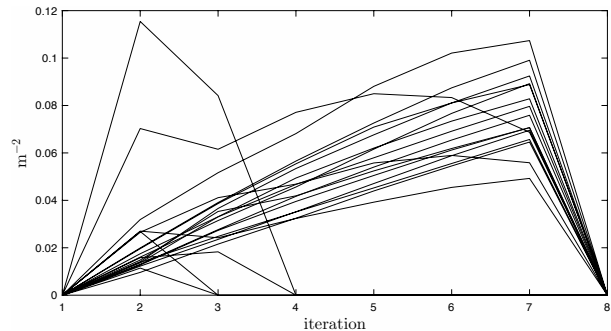}
\caption{
Plotting the amount of constraint violation of the original nonlinear constraint that will occur because we solve a linearized version.
Twenty lines are drawn, one for each landmark in a simulation with $m=20.$
The horizontal axis represents the outer loop iterations $t$. 
Lines are drawn to guide the eye
}
\vspace{-0.7cm}
\label{f:viols}
\end{figure}

Figure~\ref{f:viols} measures the amount of constraint violation of the original nonlinear constraints for the solution obtained from the linearized version that we observed in a simulation with twenty landmarks.
For the $\ccalU$ and the parameters for the function $Q$ used in these simulations, the constraint violation was very small, and driven to zero typically for $t > 10.$ 
We tested larger feasible sets $\ccalU$ and (predictably) observed a general increase in the constraint violation, although we still observed reasonable robot behaviors.

%% file: text-16p/conc.tex
In this paper, we addressed the problem of controlling a network of mobile sensors so that a set of hidden states are estimated up to a user-specified accuracy.
We proposed a new algorithm that combines the ICF with control.
The first major contribution is the capability of controlling worst-case error in the hidden states.
The second major contribution is the novel numerical technique, which is robust toward noisy gradients, computationally light, can handle large systems of coupled LMIs.
We showed that our method converges to the optimal solution and presented simulations for sparse landmark localization, where the team was able to achieve the desired estimation tolerances while exhibiting interesting emergent behaviors, notably ``dividing and conquering" the distributed estimation task, observing landmarks from diverse viewing angles.

%% file: text-16p/proofs1-2.tex
To prove Lemmas~\ref{lem:summable_constraints}~and~\ref{lem:summable_grads}, consider the following.
\begin{prop}[Summable Disagreement in Constraints]\label{prop:summable}
Let Assumptions
\ref{assume:Q} - \ref{assume:network}
hold.
For any bounded sequence $\set{\bbz_{s,k}}$ and any $i \in \ccalI,$
\[
\sum\nolimits_{k=0}^\infty \norm{ h \left(\bbz_{s,k}; \ccalD_{s,k}, \omega_{s,k} \right) - h \left(\bbz_{s,k}; \ccalD, \omega_{s,k} \right)  }_F<\infty.
\]\end{prop}
\begin{proof}
In this proof, we will essentially be comparing the function $h$ for the unconverged local data $\ccalD_{s,k}$ and the consensus data $\ccalD.$
\input{reviews-response/explain-indices-refs}
\input{reviews-response/indices}

Let us fix an arbitrary $i \in \ccalI$, and prove the proposition for the sequence $\omega_{k} = \set{i, i, i, \dots}.$
If the proof holds for this sequence, it thus holds for the random sequences $\omega_{k} \subset \ccalI$ because
\agn*{
\sum\nolimits_{k=0}^\infty &\norm{ h \left(\bbz_{k}; \ccalD_{k}, \omega_{k} \right) - h \left(\bbz_{k}; \ccalD, \omega_{k} \right)  }_F
\\
&\leq 
\sum\nolimits_{i \in \ccalI} \sum_{k=0}^\infty \norm{ h \left(\bbz_{k}; \ccalD_{k}, i\right) - h \left(\bbz_{k}; \ccalD, i\right)  }_F,
}
and if each sequence $\norm{ h \left(\bbz_{k}; \ccalD_{k}, i\right) - h \left(\bbz_{k}; \ccalD, i\right)  }_F$ is summable, so is their sum.

Recall the definition of $h$ from \eqref{eq:define_h}.
We have that 
\begin{subequations}
\label{eq:nu_terms}
\begin{align}
&h(\bbz_{k}; \ccalD_{k}, i ) 
-
h
\left(\bbz_{k}; \ccalD,i \right)=
\label{eq:inv_sigmas}
S_{i,k} -
S_i 
\\
\label{eq:Qminus}
&\;\;+\sum\nolimits_{v \in\ccalS} Q \left(\bbr_v, \hbbx_{i,k}\right) -Q \left(\bbr_v,  \hbbx_i \right)
\\
\label{eq:diffQ}
&\;\;\;+
\sum_{j=1}^q
\left(
\nabla_j Q(\bbr_v, \hbbx_{i, k} )
-
\nabla_j Q(\bbr_v, \hbbx_i )
 \right) [ \bbu_v ]_j,
\end{align}
\end{subequations}
where $S_i$ and $\hbbx_i$ are the eventual consensus variables that compose the consensus data $\ccalD,$ i.e., $S_i = S_{i,\infty}$ and $\hbbx_i = \hbbx_{i,\infty}.$
The remainder of the proof will show that each of the constituent terms, specifically, \eqref{eq:inv_sigmas}, \eqref{eq:Qminus}, and \eqref{eq:diffQ}, are summable, and the proof will follow from that immediately.

\noindent {\bf \eqref{eq:inv_sigmas}}
For this term, recall from line~\ref{line:D} of Algorithm~\ref{alg:filter+opt} that $S_{i,k}\gets n \Xi_{i,k}$, where $n$ is the number of agents in the network.
It is well known that for the consensus sequence $\set{\Xi_{i,k}}_{k \in \mbN}$ the errors $\set{ \Xi_{i,k} -\Xi_{i,\infty}}_{k \in \mbN}$ are summable; cf., e.g., \cite{ANquan2008}.
Noting that $S_i = n\Xi_{i, \infty}$ and $S_{i,k} = n\Xi_{i, k},$ we have the result.

\noindent {\bf \eqref{eq:Qminus}}
By Assumption~\ref{assume:Q},
$
\norm{ \nabla_{q + j} Q(\bbr_v, \bbx_i )}_F \le \eta_1
$
for all $\bbr_v$, $\bbx_i$ and $j$.
Then, for each $v \in \ccalS$ and $k \in \mbN,$ using the taylor expansion of $Q(\bbr_v, \cdot)$ about the point $\hbbx_i$, we have that
\agn*{
&\norm{ 
Q \left(\bbr_v, \hbbx_{i,k} \right) - Q \left(\bbr_v, \hbbx_i \right) 
}_F
=
\\
&=
\norm{
\sum\nolimits_{j=1}^p \nabla_{q+j} Q(\bbr_v, \hbbx_i) [\hbbx_{i,k} \! - \! \hbbx_i ]_j
\! + \!
o \left(
\norm{
\hbbx_{i,k} - \hbbx_i
}
\right)
}_F
\\
&\leq \eta_1 \norm{
\hbbx_{i,k} - \hbbx_i
}_1 + 
o \left(
\norm{
\hbbx_{i,k} - \hbbx_i
}
\right),
}
which is summable because $\norm{
\hbbx_{i,k} - \hbbx
}$ is the error sequence of a consensus filter, which is summable.

\noindent {\bf \eqref{eq:diffQ}}
First, since $\bbu_v <\delta$ for all $v \in \ccalS,$ for any $j = 1, \dots, p$, we have that
\begin{align}\label{eq:firstdelta}
\norm{
	\left(
		\nabla_j Q(\bbr_s, \hbbx_{i, k} )
		-
		\nabla_j Q(\bbr_s, \hbbx_i )
	 \right)[ \bbu_s ]_j 
	}_F
	\\ \nonumber
	\le \delta \norm{
		\nabla_j Q(\bbr_s, \hbbx_{i, k} )
		-
		\nabla_j Q(\bbr_s, \hbbx_i )
	}_F 
\end{align}
Also due to Assumption~\ref{assume:Q}, $\eta_2$ is the upper bound for the norm of the second derivative of $Q$.
In particular, let 
$\nabla_{q + \ell}  \nabla_j Q  \left(\bbr_v , \hbbx_i \right)$
denote the mixed second partial derivative of $Q$ with respect to the $\ell$-th component of the $\bbx_i$ coordinates and the $j$-th component of the $\bbr_v$ coordinates of the function $Q$, evaluated at the point $(\bbr_v, \hbbx_i).$
This notation is useful for the following inequality, which uses the first order Taylor expansion with respect to the $\bbx_i$ coordinates of $  \nabla_j Q $.
In particular, the norm on the right hand side of \eqref{eq:firstdelta} is bounded above by
\agn*{
	&\norm{
			\sum\nolimits_{\ell=1}^p
					\nabla_{q + \ell}  \nabla_j Q
				 \left(\bbr_v , \hbbx_i \right)
				 [
				\hbbx_{i,k} \!-\! \hbbx_i]_\ell
			+
			o \left(
				\norm{
					\hbbx_{i,k} \!-\! \hbbx_i
				}
			\right)
	}_F
	\\
	&\qquad \leq    
	\left( \eta_2 
		\norm{
			\hbbx_{i,k} \! -\! \hbbx_i
		}_1+
		o \left(
			\norm{
				\hbbx_{i,k} \! -\! \hbbx_i
			}
		\right)
	\right).
}
The sequence above is summable again because $\norm{
\hbbx_{i,k} - \hbbx_i
}$ is the error sequence of a consensus filter, which is summable.

Since \eqref{eq:inv_sigmas}, \eqref{eq:Qminus}, and \eqref{eq:diffQ} are each summable, their sum is also a summable sequence, and the proof follows.\end{proof}

\begin{proof}[Proof of Lemma~\ref{lem:summable_constraints}]
Suppose $\omega_{s,k} = \set{i,i,i,\dots}$.
If the proof holds for this arbitrary $i$, then it holds for arbitrary sequences $\omega_{s,k} \subset \ccalI$ for the same reasons as in Proposition~\ref{prop:summable}.
Let $X_k = h \left(\bbz_{s,k}; \ccalD_{s,k}, i \right) $ and $Y_k = h \left(\bbz_{s,k}; \ccalD, i \right)$.
From the definition of $\nu_{s,k},$
\agnsub{
\abs{\nu_{s,k} }&=\abs{
h_+ \left(\bbz_{s,k}; \ccalD_{s,k}, i \right) - h_+ \left(\bbz_{s,k}; \ccalD, i \right) }
\\
&=\abs{ \norm{ \proj_+ X_k }_F  - \norm{\proj_+ Y_k }_F  }\label{eq:transform_var}
\\
&\leq \norm{ \proj_+ X_k   - \proj_+ Y_k }_F  \label{eq:rev_tri} \\
&\leq \norm{ X_k   - Y_k }_F   \label{eq:conv_proj_in},
}
which is summable by Proposition~\ref{prop:summable}.
In the above manipulations, \eqref{eq:transform_var} is less than or equal to \eqref{eq:rev_tri} by the reverse triangle inequality and \eqref{eq:rev_tri} is less than or equal to \eqref{eq:conv_proj_in} because $\Sym_+(p, \reals)$ is a convex set.
Thus, $\abs{\nu_{s,k} }$ is element-wise less than or equal to a summable sequence, and the proof follows.
\end{proof}

\begin{cor}[Bound for $\abs{\nu_{s,k}}$]
\label{cor:N}
The constant $N$, which is the bound for $\abs{\nu_{s,k}}$, is given by 
$
N=\frac{1}{n}
\max_{v \in \ccalS} \norm{ (\Xi_{i, 0} )_v}_F
+ m ( \eta_0 +2 \eta_1 q \delta).
$
\end{cor}
\begin{proof}
From the proof of Lemma~\ref{lem:summable_constraints}, each term $\abs{\nu_{s,k}}$ is bounded above by $\norm{ X_k   - Y_k }_F,$
which, in view of \eqref{eq:nu_terms}, is the sum of three terms.
The corollary follows by bounding each term, then taking the sum of the bounds.
We have that
\agnsub{
\abs{\nu_{s,k}} &= \abs{h_+ \left(\bbz_{s,k}; \ccalD_{s,k}, \omega_{s,k} \right) - h_+ \left(\bbz_{s,k}; \ccalD, \omega_{s,k} \right)} \nonumber
\\
&\le  \norm{ X_k   - Y_k }_F \nonumber
\\
&\leq
\norm{S_{i,k} - S_i }_F \label{eq:info_seq}
\\
&\,+\!\sum\nolimits_{v \in\ccalS} Q \left(\bbr_v, \hbbx_{i,k}\right) -Q \left(\bbr_v,  \hbbx_i \right) \label{eq:Q_seq}
\\
&\,+\!
\sum\nolimits_{j=1}^q
(
\nabla_j Q(\bbr_s, \hbbx_{i, k} )
\! -\!
\nabla_j Q(\bbr_s, \hbbx_i )
)\! [ \bbu_s ]_j.
\label{eq:dQ_seq}
}
The first term \eqref{eq:info_seq} is the error sequence of a consensus filter, which is bounded by the largest initial condition in the network $\frac{1}{n}\max_{v \in \ccalS} \norm{ (\Xi_{i, 0} )_v}_F$.
For the second term \eqref{eq:Q_seq}, recall Assumption~\ref{assume:Q} bounded $\norm{ Q}_F$ by $\eta_0.$ 
Given that $Q$ is positive semidefinite, \eqref{eq:Q_seq} is bounded by $ m \eta_0.$
Similarly, the third term is bounded by $2 \eta_1 q \delta$, and, by taking the sum, the proof follows.
%
%
%
\end{proof}

The proof of Lemma~\ref{lem:summable_grads} requires the help of several intermediate results.
First, we present two basic results in classical analysis.
Both are direct results of the Implicit Function Theorem; see, e.g., \cite{marsden93}.
\begin{cor}
\label{thm:locus}
Let $p \colon \reals^n \to \reals$ be a nonzero polynomial function.
Then, the set of roots of $p$ has measure zero in $\reals^n.$
\end{cor}
\begin{cor}
\label{thm:measure}
Let $f \colon \reals^{n+m} \to \reals^m$ be a smooth function, $f'$ surjective, and $A \subset \reals^m$ a measure zero set.  Then, the preimage of $A$ under $f$ is measure zero in $\reals^{n+m}.$
\end{cor}

\begin{prop}[Set of Reachable Information Matrices that Are Singular or Have Nonsimple Eigenvalues is Measure Zero]
\label{prop:measure}
Let 
\agnsub{
\ccalZ_1 &= \set{X \in \Sym (p, \reals) \mid \lambda_i (X) = \lambda_j (X) , i\neq j}, 
\\
\ccalZ_2&= \set{X \in \Sym (p, \reals) \mid \det (X) = 0}.
}
Now, recall the function $h$ defined in \eqref{eq:define_h} and fix arbitrary $i \in \ccalI.$
The set
\begin{equation}\label{eq:mzero-sets}
\ccalZ = \set{ (\bbz, \ccalD) \mid h (\bbz; \ccalD , i) \in \ccalZ_1 \cup \ccalZ_2 }
\end{equation}
is measure zero.
\end{prop}

\begin{proof}
Both the determinant and the discriminant of the characteristic polynomial are nonvanishing polynomials on $\Sym(p,\reals).$
Thus, $\ccalZ_1$ and $\ccalZ_2$ represent the zero locii of their respective polynomials, and they are measure zero sets by Corollary~\ref{thm:locus}.


To complete the proof, we need only to check surjectivity of $h'$ and apply Theorem~\ref{thm:measure}.
Recall the first definition of $h$ with respect to the constant $\bbr$ and the variables $\gamma, S, \hbbx,$ and $ \bbu$, which were later shortened to the inputs $\bbz$ and the data $\ccalD.$
Omitting unnecessary subscripts from the original definition in \eqref{eq:define_h}, slightly abusing notation, we have
\[
h (\gamma, S,  \hbbx_i, \bbu)
\! = \!
\gamma I - S- \sum_{s\in \ccalV} \! Q(\bbr_s, \hbbx_i) - \! \sum_{j=1}^q \nabla_j Q(\bbr_s, \hbbx_{i}  ) [\bbu_s]_j.
\]
Note that, in this notation, 
$
h : \reals \times \Sym_{++}(p, \reals) \times \reals^{p+qn}
\to \Sym(p, \reals).
$
Since $ \Sym_{++}(p, \reals)$ is homeomorphic to $\reals^{p(p+1)/2}$, without loss of generality, we may write $h$ as a map from $\reals^{a+b} $ to $ \reals^b$ with 
$
a = 1 + p(p+1)/2 + p + qn
$
and
$
b=p(p+1)/2.
$
If we write the derivative $h' :\reals^{a+b} \to \reals^b$ as a $b \times a$ matrix $\nabla h(\gamma, S,  \hbbx_i, \bbu),$ the partial derivatives corresponding to the vectorized version of $S$, i.e., $\nabla_S h(\gamma, S,  \hbbx_i, \bbu) $ will be $I_b$, so that $\nabla h(\gamma, S,  \hbbx_i, \bbu)$ will be easily be rank $b,$ which is the required condition for surjectivity of $h'$.
%
%
\end{proof}

\begin{prop}[Bounded Entries of $\nabla h_+$]
\label{prop:bounded}
For any input $\bbz$ and data $\ccalD$, the function
$\psi_{ij}( \cdot, \bbz, \ccalD)$, defined in \eqref{eq:psi}, representing the $j$-th entry of the vector $\nabla h_+ (\bbz; \ccalD, i)$, is bounded for all $i \in \ccalI$ and $j \in \set{1, \dots, m+qn}.$\end{prop}
\begin{proof}

Recall the definition of $\psi_{ij} (\cdot,\bbz,\ccalD)\colon\Sym(p,  \reals) \to \reals$, repeated here for convenience,
\[
\psi_{ij} (X,\bbz,\ccalD)\triangleq\begin{cases}
\frac{
{{\rm Tr}} 
\left(
\nabla_j h (\bbz, \ccalD,i)
X 
\right) }{\norm{X}_F}
&{{\rm if}} \; \norm{X}_F\! >0\\
d&{{\rm else}}
\end{cases}
\]
Note that $\nabla_j h (\bbz, \ccalD,i)$ is a constant matrix for any $j$ because $h$ is a linear function of $\bbz.$
To simplify notation, call this matrix $A_{ij}$.
Let $\bblambda= [\lambda_1 \; \cdots \; \lambda_p]^\top$ denote the vector of eigenvalues of $X$ and $\bbv_j$ denote the eigenvector corresponding to $\lambda_j.$
We can write $X =  \sum_{\alpha=1}^p \lambda_\alpha \bbv_\alpha \bbv_\alpha^\top.$
If $X \neq \bb0$, then $\norm{\bblambda} \neq 0$ and
\begin{subequations}
\begin{align}
&\psi_{ij}=
\frac{{{\rm  Tr}} A_{ij} X}{ \sqrt{{{\rm Tr }} X^2}}
\\
&=
\frac{{{\rm Tr }}A_{ij} \left( \sum_{\alpha=1}^p \lambda_\alpha \bbv_\alpha \bbv_\alpha^\top \right)}{ \sqrt{ \sum_{\alpha=1}^p \lambda_\alpha^2}} =
\frac{ \sum_{\alpha=1}^p \lambda_\alpha \bbv_\alpha^\top A_{ij}\bbv_\alpha}{ \norm{\bblambda}} .\label{eq:psi_as_eigs}
\end{align}
\end{subequations}
Let $B_{ij}\triangleq \diag (v_1^\top A_{ij}v_1, \dots, v_p^\top A_{ij}v_p),$
and note that
$
\psi_{ij}=
{{\bf 1}}^\top B_{ij}\bblambda / \norm{\bblambda}.$
Note that the limit of $\psi_{ij} (X,\bbz,\ccalD)$ therefore exists for sequences $\set{X_k} \to \bb0$.
With respect to bounding the magnitude of $\psi_{ij}$, we have that
\[
\abs{ \psi_{ij} } = \frac{\abs{ {{\bf 1}}^\top B_{ij}\bblambda }}{\norm{\bblambda}}\leq 
\norm{B_{ij}{{\bf 1}}}
=
\sqrt{
 \sum\nolimits_{\alpha=1}^p  
 \left(
 \bbv_\alpha^\top A_{ij}\bbv_\alpha
 \right)^2}.
\]
Therefore, $\psi_{ij}$ is bounded.
Additionally, $\abs{ \psi_{ij} }$ is maximized if $\set{\bbv_\alpha}$ represent a set of orthogonal eigenvectors of $A_{ij}$, that is, if $X$ has the same eigenvectors as $A_{ij}$.
In this case, the bound is tight and equal to $\norm{\bbmu_{ij}},$
where $\bbmu_{ij}$ is the vector of eigenvalues of the matrix $A_{ij} \triangleq \nabla_j h (\bbz, \ccalD,i).$

It is left to show that $A_{ij}$ is bounded.  
This obvious when we write out the individual terms of $h$ from \eqref{eq:define_h}.
It can be seen that if $j\leq m,$ then $A_{ij} = I.$
Otherwise, it is the matrix of partial derivatives of $Q$, which are bounded by Assumption~\ref{assume:Q}.
\end{proof}

\begin{cor}[Bound for $\norm{\bd_{s,k}}$]
\label{cor:D}
Let $A_{ij} \triangleq \nabla_j h (\bbz, \ccalD,i).$
The constant $L_h$ and the constant $D$, which are bounds for $\norm{h'_+ (\bbz, \ccalD, i)}$ and $\norm{\bd_{s,k}}$, respectively, are given by 
$
L_h =\sqrt{\bbarmu^2_1 + \cdots + \bbarmu^2_{m + qn}},
\; \and \;
D = 2 L_h,
$
where
$
\bbarmu_j \triangleq \max_{i \in \ccalI} \set{ \norm{ \bbmu_{ij}} \mid \bbmu_{ij}}
$ is the vector of eigenvalues of $A_{ij}.$
\end{cor}
\begin{proof}
The $j$-th entry of $h'_+$ is a restriction of $\psi_{ij}$ to a subset of its domain.
Thus, bounds for $\psi_{ij}$ also hold for the $j$-th entry of $h'_+.$
In the proof of Proposition~\ref{prop:bounded}, we saw that $\psi_{ij}$ is bounded by the 2-norm of the eigenvalues of $A_{ij} \triangleq \nabla_j h (\bbz, \ccalD,i).$
Therefore, $L_h = \norm{\bbmu_{ij}}_2.$
Taking the norm of both sides of \eqref{eq:differror}, we have that
\[
\norm{\bbdelta_{s,k}}=
\norm{
h'_+\left(\bbz_{s,k}; \ccalD_{s,k}, \omega_{s,k} \right)-h'_+\left(\bbz_{s,k}; \ccalD, \omega_{s,k} \right)}.
\]
Thus, the $j$-th entry of $\bd_{s,k}$ is bounded by $2 \bar{\mu}_j$ and $\norm{\bbdelta_{s,k}}$.
Taking the norm, we retrieve the result of the Corollary.
\end{proof}

\begin{prop}[Bounded Derivatives of $\psi_{ij}$]
\label{prop:boundpartials}
Assume $X \in \Sym_+ (p,\reals).$
Then, the magnitudes of the partial derivatives
$\norm{ \nabla_\ccalD \psi_{ij}  (X, \bbz, \ccalD ) }$
and
$\norm{ \nabla_X \psi_{ij}  (X, \bbz, \ccalD ) }$
are point-wise bounded for all $i \in \ccalI$ and $j \in \set{1, \dots, m+qn}.$
\end{prop}\begin{proof}
If $\norm{X} = 0,$ then $\psi_{ij}$ is constant, and thus the magnitudes of its partial derivatives are zero.
The rest of the proof deals with $\norm{X}> 0.$

To simplify notation, we refer to $\nabla_j h (\bbz, \ccalD,i)$ by the shorthand $A_{ij}.$
Choose $\epsilon  \in \reals$ such that $ \norm{X}>\epsilon > 0.$
Fix arbitrary $i$.  
Since $X$ and $A_{ij}$ are symmetric, for any $j$ we have that
$
\nabla_X \! \psi_{ij}  (X, \bbz, \ccalD ) \!= \! \frac{1}{\norm{X}} A_{ij} - \frac{{{\rm Tr}} A_{ij}X}{\norm{X}^3}X
$
and
\begin{equation}\label{norm_psi_deriv}
\norm{\nabla_X \psi_{ij}  (X, \bbz, \ccalD ) }^2 \!=\! \norm{X}^{-4} (\norm{A_{ij}}^2 \norm{X}^2 \!\! -\!\! ( {\rm Tr} A_{ij}X)^2),
\end{equation}
which is strictly less than 1 for all $\norm{X} > \norm{A_{ij}}.$
For $\norm{X} \le \norm{A_{ij}},$
$ ( \norm{A_{ij}}/ \epsilon)^4$ is an upper bound for \eqref{norm_psi_deriv}.

For the term $\nabla_\ccalD \psi_{ij}  (X, \bbz, \ccalD  ),$ note that the only part of $\psi_{ij}$ that depends on $\ccalD$ is $A_{ij}.$
Thus by the chain rule,
\[
\nabla_\ccalD \psi_{ij}  (X, \bbz, \ccalD )=  
\left(\nabla_{A_{ij}} \psi_{ij}  (X, \bbz, \ccalD) 
\right)
\left( 
 \nabla_\ccalD A_{ij}
 \right).
\]
The first term is given by 
$
\nabla_{A_{ij}} \psi_{ij}  (X, \bbz, \ccalD)= \frac{1}{\norm{X}_F}I
\prec \frac{1}{\epsilon} I.$
For the second term, recall that $\ccalD = ( \hbbx, S).$
Split the partial derivative into two parts $\nabla_\ccalD=(\nabla_{\hbbx}, \nabla_{S})$ so that
\agn*{
\nabla_\ccalD \nabla_j h (\bbz, \ccalD,i)=  (\nabla_{\hbbx}\nabla_j h (\bbz, \ccalD,i) ,
 \nabla_{S}\nabla_j h (\bbz, \ccalD,i) ).
}
For any $j$, $\nabla_{S}\nabla_j h (\bbz, \ccalD,i) = I_p,$ and obviously $\norm{I_p} \le p.$
Similarly, $\nabla_j h (\bbz, \ccalD,i)= I_p$ for $j \le m.$
If $j > m,$ then $\nabla_{\hbbx}\nabla_j h (\bbz, \ccalD,i) $ results in second mixed partial derivatives of $Q$, which we assumed to be bounded in 
Assumption~\ref{assume:Q}.\end{proof}

\begin{prop}[Local Differentiability of Constraint Gradient]
\label{prop:neighborhood}
Consider an arbitrary $\bbz \in \reals^{m + qn}$ and $\ccalD =(\hbbx,S)$ with $\hbbx_i \in \reals^{p} $ and $S_i \in \Sym_+ (p, \reals).$
Suppose $h\left(\bbz;\ccalD, i \right)$ is nonsigular with distinct eigenvalues.
Then,
$\exists$ an open neighborhood $\Omega$ containing $h\left(\bbz;\ccalD, i \right)$ such that $\psi_{ij} \left(\proj_+ \left( \cdot\right), \bbz, \ccalD\right)\big|_\Omega$ is differentiable.
\end{prop}
\begin{proof}
{\bf Notation}
Note that $f \in C^\infty (\ccalA, \ccalB)$ means that the $n$-th derivative of $f$ exists at all points in $\ccalA$ for all $n \in \mbN.$
We also use the notation $f|_\Omega$ to denote the restriction of $f$ to $\Omega.$

Roughly speaking, the proof uses a well known result \cite{magnus85}   to show that, if $h\left(\bbz;\ccalD, i \right)$ has distinct eigenvalues, then for $\alpha=1,\dots,p$ we can find an open neighborhood $\Omega$ in which there are differentiable maps $\lambda_\alpha \in C^\infty(\Omega_\alpha, \reals)$ and $\bbv_\alpha \in C^\infty(\Omega_\alpha, \reals^p)$ that represent the $\alpha$-th eigenvalue-eigenvector pair of each $X \in \Omega$ and agree with the eigenvalues and eigenvectors of $h\left(\bbz;\ccalD, i \right)$.
Since $\psi_{ij}\left(\proj_+(\cdot), \bbz ,\Sigma \right)$ is well behaved in an open neighborhood, it is not difficult to arrive at the desired result.
The details of the proof are as follows, and generalize the line of reasoning in \cite{magnus85}.

Let $(\bbarlambda_\alpha, \bar{\bbv}_\alpha)$ be an eigenvalue-eigenvector pair of $h\left(\bbz;\ccalD, i \right).$
Since $h\left(\bbz;\ccalD, i \right) \in \Sym(p, \reals)$, we know that $\bbarlambda_\alpha$ is real valued and distinct.
Consider the mapping $\bbphi \colon \reals \times \reals^{p} \times \Sym(p, \reals) \to \reals^{p+1}$ given by
\begin{equation}
\bbphi(\lambda, \bbu, X) \triangleq \mat{(\lambda I - X) \bbv\\ \bbv^\top \bbv -1}.
\end{equation}
By definition, $\bbphi (\lambda, \bbu, X) = \bb0$ if and only if $ (\lambda,\bbv)$ are a (normalized) eigenvalue-eigenvector pair of $X$.
It is also a well known properties of eigenvalues and eigenvectors that
\[
\nabla_{(\lambda, \bbu)} \bbphi \left(\bbarlambda_\alpha, \bar{\bbv}_\alpha, h\left(\bbz;\ccalD, i \right) \right) 
\]
is full rank if and only if $\bbarlambda_\alpha$ is simple, which is true by assumption.
By the implicit function theorem, there is an open neighborhood $\Omega_\alpha$ of $h\left(\bbz;\ccalD, i \right)$ in which we can find differentiable functions $\lambda_\alpha \in C^\infty(\Omega_\alpha, \reals)$ and $\bbv_\alpha \in C^\infty(\Omega_\alpha, \reals^p)$ such that
$
\lambda_\alpha(h\left(\bbz;\ccalD, i \right)) = \bbarlambda_\alpha$
and
$\bbv_\alpha(h\left(\bbz;\ccalD, i \right)) =  \bar{\bbv}_\alpha$
Let
\begin{subequations}
\begin{align}
\Omega_\alpha^+ &= \set{X \in \Omega_\alpha \mid 
 \lambda_\alpha (X)> 0 }, \; \and \\
 \Omega_\alpha^- &= \set{X \in \Omega_\alpha \mid 
 \lambda_\alpha (X)< 0 }.
\end{align}
\end{subequations}

If $\bbarlambda_\alpha$ is positive for all $\alpha = 1, \dots, p,$ every $\Omega_\alpha^+$ is an open set that contains $h\left(\bbz;\ccalD, i \right).$
Let
\begin{equation}
\Omega = \cap_\alpha \Omega_\alpha^+,
\end{equation}
which is nonempty and open.
When restricted to $\Omega$, the mapping $\proj_+$ is the identity map, and thus (at least for this case) the proof follows if $\psi_{ij} \left(\cdot, \bbz, \ccalD \right)\big|_\Omega$ is differentiable.
To verify that it is, consider arbitrary $X \in \Omega$.
Recall again the definition of $\psi_{ij} (\cdot,\bbz,\ccalD)\colon \Sym (p, \reals) \to \reals$ from \eqref{eq:psi}.
Note that $\nabla_j h (\bbz, \ccalD,i)$ is a constant matrix because $h$ is a linear function of $\bbz.$
Call this matrix $A_{ij}$.
We have that
\[
\psi_{ij} \left(X, \bbz, \ccalD \right)
=
\frac{{{\rm Tr}} \left(A_{ij}X \right)}{\norm{X}_F}
=
\frac{ \sum_{\alpha}\lambda_\alpha(X)  \bbv_\alpha^\top(X)  A_{ij}\bbv_\alpha(X) }{ \sqrt{ \sum_{\alpha } \lambda_\alpha^2(X)}} ,
\]
which uses the same algebraic manipulation as the derivation of \eqref{eq:psi_as_eigs}.
By inspection, $\psi_{ij} (\cdot,\bbz,\ccalD)$ is differentiable at all $X \in \Omega$, and the result follows.

If $\bbarlambda_\alpha$ is negative for all $\alpha = 1, \dots, p,$ every $\Omega_\alpha^-$ is an open set that contains $h\left(\bbz;\ccalD, i \right).$
Let
\begin{align}
\Omega =\cap_\alpha \Omega_\alpha^-,
\end{align}
which is again open and nonempty.
Note that $\proj_+$ projects every matrix $X \in \Omega$ to the matrix of zeros.
The map $\psi_{ij} (\bb0,\bbz,\ccalD)$ is thus constant w.r.t. $X \in \Omega$ and is thus differentiable.

If $h\left(\bbz;\ccalD, i \right)$ has some positive and negative eigenvalues, then we can split them into two groups.
Let $\alpha$ index the positive eigenvalues and $\beta$ denote the negative eigenvalues.
Similar to before, let
\begin{align}
\Omega =\left( \cap_\alpha \Omega_\alpha^+ \right) \cup \left( \cap_\beta\Omega_\beta^- \right) ,
\end{align}
which is again nonempty, open, and contains $h\left(\bbz;\ccalD, i \right).$
Note that $\proj_+$ does not change $\lambda_\alpha$ or $\bbv_\alpha$ for matrices $X \in \Omega$, i.e., 
\begin{subequations}
\begin{align}
\lambda_\alpha (X) &= \lambda_\alpha \left(\proj_+ (X)\right) , \; \forall X \in \Omega
\\
\bbv_\alpha  (X) &= \bbv_\alpha \left(\proj_+ (X)\right)  , \; \forall X \in \Omega.
\end{align}
\end{subequations}
The eigenvalues corresponding to $\beta$ of $\proj_+ (X)$, on the other hand, are always projected to zero.
Thus we write that
$ \lambda_\beta \left(\proj_+ (X)\right) \equiv 0,$ and we have, omitting some elementary algebra, that
\[
\psi_{ij} \left(\proj_+ (X) , \bbz, \ccalD \right)
=
\frac{{{\rm Tr}} \left(A_{ij}\proj_+ (X)  \right)}{\norm{\proj_+  (X)}_F}
=
 \bbv_1^\top(X)  A_{ij}\bbv_1(X).
\]
By inspection, $\psi_{ij} \left(\proj_+ (\cdot) , \bbz, \ccalD \right)$ is differentiable at $X$.
Since $X$ was an arbitrary matrix from the set $\Omega$, the result follows.

Since $h\left(\bbz;\ccalD, i \right)$ was assumed to be nonsingular, all possibilities of combinations eigenvalues have been covered, and, for each case, we found an open set $\Omega$ in which the desired result holds.
\end{proof}

\begin{proof}[Proof of Lemma~\ref{lem:summable_grads}]
Recall the definition
\[
\bbdelta_{s,k}\triangleq
h'_+\left(\bbz_{s,k}; \ccalD_{s,k}, \omega_{s,k} \right)-h'_+\left(\bbz_{s,k}; \ccalD,  \omega_{s,k} \right),
\]
where $\set{ \omega_{s,k}}_{k \in \mbN} \subset \ccalI.$
Consider sequences of the form $ \set{i,i,i,\dots}.$
We will use the fact that
\[
\sum_{k\in \mbN} \norm{ \bbdelta_{s,k} } \leq \sum_{k\in \mbN}\sum_{i \in \ccalI}
\norm{
h'_+\left(\bbz_{s,k}; \ccalD_{s,k},i\right)-h'_+\left(\bbz_{s,k}; \ccalD,i \right)
}
\]
to show that
that $\bbdelta_{s,k}$ is summable, provided that each sequence $
\norm{
h'_+\left(\bbz_{s,k}; \ccalD_{s,k},i\right)-h'_+\left(\bbz_{s,k}; \ccalD,i \right)
}$ is summable.

To show that $h'_+\left(\bbz_{s,k}; \ccalD_{s,k},i\right)-h'_+\left(\bbz_{s,k}; \ccalD,i \right)$ is summable, we will show that the sequence of the entries for an arbitrary entry $j \in \set{1, \dots, m + qn}$ is summable.
If this is true, the norm of the whole vector will be summable as well because $\norm{\cdot}_a \geq \norm{\cdot}_b.$ if $a \leq b.$
Fix $j \in \set{1, \dots, m+ qn}$.

\noindent{\bf Case 1}
If $j\le m,$ then this entry corresponds to the variable from the original problem that we denoted as $\gamma_j$.
Note that, for this choice of $j$, we have that
\begin{equation}
\nabla_j h (\bbz, \ccalD,i)= \begin{cases}
I_p &\text{if } i=j
\\
\bb0 &\text{else}
\end{cases}
\end{equation}
for all possible variables $\bbz$ or data $\ccalD.$
Therefore, the $j$-th entry of $h'_+\left(\bbz_{s,k}; \ccalD_{s,k},i\right)-h'_+\left(\bbz_{s,k}; \ccalD,i \right)$, which is the difference of two functions
\begin{align}
\nabla_j h_+ (\bbz_{s,k},\ccalD_{s,k},i) -  \nabla_j h_+ (\bbz_{s,k}, \ccalD,i )
= 0
\end{align}
because these functions are actually the same constant.

\noindent{\bf Case 2}
If $j>m$, then we are looking at an entry of the gradient with respect to the control variables $\bbu_s,$ which are non trivial.
In this case, recall from the definition in \eqref{eq:psi} that
$
\nabla_j h_+ (\bbz, \ccalD,i)
 =
  \psi_{ij} \left(\proj_+ \left(  h\left(\bbz;\ccalD,  i\right) \right), \bbz, \ccalD\right).
$
In other words, we can look at $ \nabla_j h_+$ as essentially a composition of three functions: $\psi_{ij}, \proj_+, \; \and h.$
In the preceding lemmas and propositions, we have established useful facts about each of these constituent functions that will enable us to complete the proof.

Let $X_k = h\left(\bbz_{s,k};\ccalD_{s,k}, i \right)$ and $Y_k =  h\left(\bbz_{s,k};\ccalD, i \right),$ and note that $X_k \to Y_k$ because $h$ is continuous and $\ccalD_{s,k} \to \ccalD.$
By Proposition~\ref{prop:measure}, both sequences $\set{\bbz_{s,k},\ccalD_{s,k}}_{k \in \mbN}$ and $\set{\bbz_{s,k}, \ccalD}_{k \in \mbN}$ almost surely lie in the preimage of the set of nonsingular matrices that have distinct eigenvalues with respect to the mapping $h(\cdot; \cdot, i).$
Let $\Omega_k$ be an open neighborhood of $Y_k$ such that $\psi_{ij} \left(\proj_+ \left( \cdot\right), \bbz, \ccalD\right)\big|_{\Omega_k}$ is differentiable.
Such a neighborhood is guaranteed to be available by Proposition~\ref{prop:neighborhood} 
as long as $Y_k$ 
converges, and the limit point is nonsingular and does not have a simple spectrum.
We know that $Y_k$ converges by construction of the decreasing step sizes in Algorithm~\ref{alg:filter+opt}, and w.p.1. this limit point will be nonsingular and have a nonsimple spectrum by Proposition~\ref{prop:measure}.
Then, we can find $K \in \mbN$ such that $\forall \ell \geq K,$ $X_\ell \in \Omega_\ell.$

By Proposition~\ref{prop:boundpartials}, the derivatives of $ \psi_{ij} \left( \proj_+\left(  X\right), \bbz, \ccalD\right)$ with respect to both $X$ and $\ccalD$ are bounded.
Denote the larger bound by $\mu_\ell$.
Proposition~\ref{prop:bounded}  bounds $\psi_{ij}$ globally almost surely.
We have that
\begin{align*}
&h'_+\left(\bbz_{s,k}; \ccalD_{s,k},i\right)-h'_+\left(\bbz_{s,k}; \ccalD,i \right) = \\
 &=\psi_{ij} \left(\proj_+  \left(  
h\left(\bbz_{s,\ell}; \ccalD_{s,\ell}, i \right)
 \right), \bbz_{s,\ell}, \ccalD_{s,\ell}\right)
 \\
&\qquad\qquad\qquad-
 \psi_{ij} \left(\proj_+  \left(  
h\left(\bbz_{s,\ell}; \ccalD, i \right)
 \right), \bbz_{s,\ell}, \ccalD\right)
\\
&=
 \psi_{ij} \left(\proj_+  \left(  
X_\ell
 \right), \bbz_{s,\ell}, \ccalD_{s,\ell}\right)
-
 \psi_{ij} \left(\proj_+  \left(  
Y_\ell
 \right), \bbz_{s,\ell}, \ccalD\right)
\\
&= \!
\left\langle  
	 \nabla_X  (\psi_{ij} \! \circ \! \proj_+ )  (Y_\ell , \bbz_{s, \ell} ,\ccalD)   , X_\ell \! - \! Y_\ell 
\right\rangle_{\Sym (p, \reals)}
 \! \!\! \\
 & \qquad +
 o ( \norm{X^\ell \! - \! Y^\ell} )
\\
&\leq 
\norm{ \nabla_X \! (\psi_{ij} \! \circ \! \proj_+ \! ) (Y_\ell , \bbz_{s, \ell} ,\ccalD)  }\norm{ X_\ell \! - \! Y_\ell} + o \left( \norm{X_\ell - Y_\ell} \right)
\\
&\leq
\mu_\ell\norm{ X_\ell - Y_\ell} + o \left( \norm{X_\ell - Y_\ell} \right),
\end{align*}
where $\langle \cdot, \cdot \rangle_{\Sym(p, \reals)}$ denotes the matrix inner product.
The final term is summable if and only if $\norm{ X_\ell - Y_\ell}$ is summable.
To verify that it is, note that
\begin{align}
\norm{X_\ell - Y_\ell}
&=
\norm{ h\left(\bbz_{s,k};\ccalD_{s,k}, i \right) -
 h\left(\bbz_{s,k};\ccalD, i \right)}
\end{align}
which is summable by Proposition~\ref{prop:summable}.\end{proof}

%% file: reviews-response/explain-indices-refs.tex
Note that at some iteration $k \in \mbN,$ the ICF is actually computing the information summary $(\bbxi_{1,s,k}, \dots, \bbxi_{m,s,k}, \Xi_{1,s,k}, \dots, \Xi_{m,s,k})$, which is used to compute $\ccalD_{s,k}$ in line~\ref{line:D} of Algorithm~\ref{alg:filter+opt}.

%% file: reviews-response/indices.tex
In particular, we will need to refer to the unconverged mean and information matrix, which is
\agn*{
\ccalD_{s,k} &\triangleq  
\Big( \big((n\Xi_{1, s,k})^{-1} \bbxi_{1, s,k}, \dots, (n\Xi_{m,s,k})^{-1} \bbxi_{m,s,k}\big), \\
&\hspace{2cm}\left( n\Xi_{1, s,k} , \dots, n\Xi_{m,s,k} \right)\Big).
}
In the proof, we will drop the reference to the sensor $s$, as we only consider one sensor at a time.
In other words, the unconverged data will be written as 
\[
\! \ccalD_k 
\!\! =  \!\!
\Bigg(  \! \! \!
\bigg( \!\!\underbrace{(n\Xi_{1, k})^{ \! - \!1} \! \bbxi_{1,k}}_{\hbbx_{1,k}}, \dots,\!  \underbrace{(n\Xi_{m,k})^{\! - \! 1} \! \bbxi_{m,k}}_{\hbbx_{m,k}}
\! \bigg) \! , \!\!\bigg(\!\! \underbrace{n\Xi_{1,k}}_{S_{1,k}} , \dots,\! \underbrace{ n\Xi_{1,m}}_{S_{m,k}} \!\!\!\bigg)
\! \! \! \Bigg)
\]
for $k \in \mbN,$ i.e., this is implicitly the data local to sensor $s$ in the proof.
For the consensus data $\ccalD_\infty = \left(\left(\hbbx_{1, \infty}, \dots, \hbbx_{m, \infty}\right), \left(S_{1, \infty}, \dots, S_{m, \infty}\right)\right)$, we also drop the $k$ subscript, i.e., we write the consensus data as $\ccalD =(\hbbx, S)=  \left(\left(\hbbx_{1}, \dots, \hbbx_{m}\right),\left( S_{1}, \dots, S_{m}\right)\right) $.

%% file: text-16p/proofs2.tex

\begin{proof}[Proof of Lemma \ref{lem:basic}]
From Line \ref{line:z} of Algorithm \ref{alg:filter+opt}, the following chain of relations hold for any feasible point $\zb \in \Xc \subseteq \Xc_0$ \textit{a.s.}:
\begin{subequations}
\begin{align}
&\|\zb_{s,k}-\zb\|^2 \le  \|\vb_{s,k}-\zb \! - \! \b_{s,k}h'_+(\vb_{s,k};\ccalD_{s,k},\o_{s,k})\|^2\label{eqn:basica}\\
&~\le  \|\vb_{s,k}-\zb\|^2 \! - \! 2\b_{s,k} \la h'_+(\vb_{s,k};\ccalD_{s,k},\o_{s,k}), \vb_{s,k}-\zb \ra \nonumber\\
&~+ \b_{s,k}^2 \|h'_+(\vb_{s,k};\ccalD_{s,k},\o_{s,k})\|^2\label{eqn:basicb}\\
&~\le \|\vb_{s,k}-\zb\|^2 \! - \! 2\b_{s,k} \la h'_+(\vb_{s,k};\ccalD,\o_{s,k}), \vb_{s,k} \! - \! \zb \ra \label{eqn:basicc} \\
&~~+\! 2\b_{s,k} \|\bd_{s,k}\|\|\vb_{s,k} \! - \! \zb \|
\! + \! \b_{s,k}^2 \|h'_+(\vb_{s,k};\ccalD_{s,k},\o_{s,k})\|^2 \nonumber \\
&~\le \|\vb_{s,k}-\zb\|^2 \! - \! 2\b_{s,k} h_+(\vb_{s,k};\ccalD,\o_{s,k})\label{eqn:basicd}\\
&~+ 2\b_{s,k} \|\bd_{s,k}\|\|\vb_{s,k} \! - \! \zb \|
+ \b_{s,k}^2 \|h'_+(\vb_{s,k};\ccalD_{s,k},\o_{s,k})\|^2,\nonumber
\end{align}
\end{subequations}
where (\ref{eqn:basica}) follows from the nonexpansiveness of the projection operator; (\ref{eqn:basicb}) follows from the expansion of $\|\cdot\|^2$; (\ref{eqn:basicc}) follows from relation (\ref{eq:differror}) and the Schwarz inequality;
and (\ref{eqn:basicd}) follows from the convexity of the function $h_+(\cdot;\cdot,\o_{s,k})$, and $h_+(\zb;\ccalD,\o_{s,k})=0$ due to the feasibility of the point $\zb$, i.e.,
$
h_+(\vb_{s,k};\ccalD,\o_{s,k})
\le \la h'_+(\vb_{s,k};\ccalD,\o_{s,k}), \vb_{s,k}-\zb \ra.
$
From relation (\ref{eq:constrainterror}) and the definition of $\b_{s,k}$ in (\ref{eqn:beta}), we further have
\begin{align*}
&- 2\b_{s,k} h_+(\vb_{s,k};\ccalD,\o_{s,k})+ \b_{s,k}^2 \|h'_+(\vb_{s,k};\ccalD_{s,k},\o_{s,k})\|^2\\
&~= - 2\b_{s,k} h_+(\vb_{s,k};\ccalD_{s,k},\o_{s,k}) + 2\b_{s,k}|\nu_{s,k}| \\
&~~+ \b_{s,k}^2 \|h'_+(\vb_{s,k};\ccalD_{s,k},\o_{s,k})\|^2\\
&~= 2\b_{s,k}|\nu_{s,k}| - \frac{h^2_+(\vb_{s,k};\ccalD_{s,k},\o_{s,k})}{\|h'_+(\vb_{s,k};\ccalD_{s,k},\o_{s,k})\|^2}\\
&~\le 2\b_{s,k}|\nu_{s,k}| - \frac{h_+^2(\vb_{s,k};\ccalD_{s,k},\o_{s,k})}{(L_h+D)^2},
\end{align*}
where the last inequality follows from relation \eqref{eqn:L_h1}.
Combining this with inequality (\ref{eqn:basicd}), we obtain for any $\zb \in \Xc$ \textit{a.s.}:
\begin{align}\label{eqn:basic1}
\|\zb_{s,k}-\zb\|^2
\le &~ \|\vb_{s,k}-\zb\|^2
+ 2\b_{s,k} \|\bd_{s,k}\|\|\vb_{s,k}-\zb \|\\
&~+2\b_{s,k}|\nu_{s,k}|
- \frac{h_+^2(\vb_{s,k};\ccalD_{s,k},\o_{s,k})}{(L_h+D)^2}.\nonumber
\end{align}
For the last term on the right-hand side, we can rewrite
$
h_+(\vb_{s,k};\ccalD_{s,k},\o_{s,k}) = h_+(\vb_{s,k};\ccalD_{s,k},\o_{s,k}) -h_+(\pb_{s,k};\ccalD_{s,k},\o_{s,k})
+h_+(\pb_{s,k};\ccalD_{s,k},\o_{s,k}).
$
Therefore,
\begin{align}\label{eqn:gfnpre}
&h_+^2(\vb_{s,k};\ccalD_{s,k},\o_{s,k})\ge 2h_+(\pb_{s,k};\ccalD_{s,k},\o_{s,k})\nonumber\\
&\hspace{2cm}\times\left(h_+(\vb_{s,k};\ccalD_{s,k},\o_{s,k}) -h_+(\pb_{s,k};\ccalD_{s,k},\o_{s,k})\right)\nonumber\\
&\hspace{3cm} + h_+^2(\pb_{s,k};\ccalD_{s,k},\o_{s,k}).
\end{align}
Regarding the first term on the right-hand side of (\ref{eqn:gfnpre}), the following chain of relations holds:
\begin{subequations}\label{eqn:basice}
\begin{align}
&2h_+(\pb_{s,k};\ccalD_{s,k},\o_{s,k})\nonumber\\
&~~\times\left(h_+(\vb_{s,k};\ccalD_{s,k},\o_{s,k}) -h_+(\pb_{s,k};\ccalD_{s,k},\o_{s,k})\right)\nonumber\\
&~\ge-2h_+(\pb_{s,k};\ccalD_{s,k},\o_{s,k})\nonumber\\
&~~\times\left|h_+(\vb_{s,k};\ccalD_{s,k},\o_{s,k}) -h_+(\pb_{s,k};\ccalD_{s,k},\o_{s,k})\right|\label{eqn:basicf}\\
&~\ge -2(L_h+D)\|\vb_{s,k}-\pb_{s,k}\|h_+(\pb_{s,k};\ccalD_{s,k},\o_{s,k})\label{eqn:basicg}\\
&~\ge -2\a_k (L_h+D)L_fh_+(\pb_{s,k};\ccalD_{s,k},\o_{s,k})\label{eqn:basich}\\
&~\ge -\tau \a_k^2(L_h+D)^2L_f^2 - \frac{1}{\tau}h_+^2(\pb_{s,k};\ccalD_{s,k},\o_{s,k}),\label{eqn:basici}
\end{align}
\end{subequations}
where (\ref{eqn:basicf}) follows from the Schwarz inequality; (\ref{eqn:basicg}) follows from relation (\ref{eqn:L_h2}); (\ref{eqn:basich}) follows from Line \ref{line:v} of Algorithm \ref{alg:filter+opt}; and (\ref{eqn:basici}) follows from using 
$
2|a||b| \le \tau a^2 + \frac{1}{\tau} b^2
$
with $a=\a_k (L_h+D)L_f$ and $b=h_+(\pb_{s,k};\ccalD_{s,k},\o_{s,k})$, and $\tau >0$ is arbitrary.
Using relations (\ref{eqn:gfnpre})-(\ref{eqn:basice}) in (\ref{eqn:basic1}), we obtain for all $\zb \in \Xc$ \textit{a.s.}:
\begin{align}\label{eqn:basic2}
&\|\zb_{s,k}-\zb\|^2 \le \|\vb_{s,k}-\zb\|^2
+ 2\b_{s,k} \|\bd_{s,k}\|\|\vb_{s,k}-\zb \|\\
&~+2\b_{s,k}|\nu_{s,k}|-\frac{\tau-1}{\tau (L_h+D)^2}h_+^2(\pb_{s,k};\ccalD_{s,k},\o_{s,k})
+\tau \a_k^2 L_f^2.\nonumber
\end{align}

Similarly to (\ref{eqn:basica})-(\ref{eqn:basicd}), from Line \ref{line:v} of Algorithm \ref{alg:filter+opt}, the following chain of relations hold for any $\zb\in \Xc \subseteq \Xc_0$ \textit{a.s.}:
\begin{subequations}
\begin{align}
&\|\vb_{s,k}-\zb \|^2\le \|\pb_{s,k}-\zb-\a_kf'(\pb_{s,k})\|^2\label{eqn:basicj}\\
&~\le \|\pb_{s,k}-\zb\|^2 - 2\a_k\la f'(\pb_{s,k}), \pb_{s,k}-\zb\ra + \a_k^2\|f'(\pb_{s,k})\|^2\label{eqn:basick}\\
&~\le \|\pb_{s,k}-\zb\|^2 - 2\a_k(f(\pb_{s,k})-f(\zb)) + \a_k^2L_f^2,\label{eqn:basicm}
\end{align}
\end{subequations}
where (\ref{eqn:basicj}) follows from the nonexpansive projection;
(\ref{eqn:basick}) follows from the expansion of $\|\cdot\|^2$;
(\ref{eqn:basicm}) follows from the convexity of the function $f$ and relation \eqref{eqn:L_f1}.
For the second term on the right-hand side, we further have for any $\check{\zb} \in \Xc$:
\begin{subequations}
\begin{align}
&- 2\a_k(f(\pb_{s,k})-f(\zb)) \nonumber\\
&~=  -2\a_k(f(\pb_{s,k}) - f(\check{\zb}) + f(\check{\zb})-f(\zb))\label{eqn:basicn}\\
&~\le  - 2\a_k \la f'(\check{\zb}), \pb_{s,k} - \check{\zb}\ra - 2\a_k(f(\check{\zb})-f(\zb))\label{eqn:basico}\\
&~\le  2\a_k L_f \|\pb_{s,k} - \check{\zb}\| - 2\a_k(f(\check{\zb})-f(\zb))\label{eqn:basicp}\\
&~\le  4\eta\a_k^2 L_f^2 + \frac{1}{4\eta}\|\pb_{s,k} - \check{\zb}\|^2 - 2\a_k(f(\check{\zb})-f(\zb)),\label{eqn:basicq}
\end{align}
\end{subequations}
where (\ref{eqn:basicn}) follows from adding and subtracting $f(\check{\zb})$;
(\ref{eqn:basico}) follows from the convexity of the function $f$;
(\ref{eqn:basicp}) follows from the Schwarz inequality;
and (\ref{eqn:basicq}) follows from using 
$
2|a||b| \le 4\eta a^2 + \frac{1}{4\eta} b^2
$
with $a = \a_kL_f$ and $b = \|\pb_{s,k} - \check{\zb}\|$.
Combining (\ref{eqn:basicq}) in (\ref{eqn:basicm}), we have for any $\zb, \check{\zb} \in \Xc \subseteq \Xc_0$ \textit{a.s.}:
\begin{align*}
\|\vb_{s,k}-\zb \|^2
\le&~ \|\pb_{s,k}-\zb\|^2 - 2\a_k(f(\check{\zb})-f(\zb)) \\
&~+ \frac{1}{4\eta}\|\pb_{s,k} - \check{\zb}\|^2 + (1+4\eta)\a_k^2L_f^2.
\end{align*}
Substituting this inequality in relation (\ref{eqn:basic2}), we obtain
\begin{align*}
&\|\zb_{s,k}-\zb\|^2
\le  \|\pb_{s,k}-\zb\|^2 - 2\a_k(f(\check{\zb})-f(\zb))\\
&~+ 2\b_{s,k} \|\bd_{s,k}\|\|\vb_{s,k}-\zb \|+2\b_{s,k}|\nu_{s,k}|\\
&~ + \frac{1}{4\eta}\|\pb_{s,k} - \check{\zb}\|^2  + A_{\eta,\tau}\a_k^2L_f^2\\
&~-\frac{\tau-1}{\tau(L_h+D)^2}h^2_+(\pb_{s,k};\ccalD_{s,k},\o_{s,k}),
\end{align*}
where $A_{\eta,\tau} = 1+4\eta+\tau$.
From relation \eqref{eqn:C_z}, we have 
$
\|\vb_{s,k}-\zb \|\le C_{\zb}.
$
Using this and the upper estimate for $\b_{s,k}$ in (\ref{eqn:betabnd}) for bounding the three error terms completes the proof.
\end{proof}

\begin{proof}[Proof of Lemma \ref{lem:feas}]
We use Lemma~\ref{lem:basic} with $\check{\zb}=\zb = \mathsf{\Pi}_{\mathcal{X}}[\pb_{s,k}]$.
Therefore, for any $s\in \ccalS$, and $k \ge 1$, we obtain \textit{a.s.}:
\begin{align}\label{eqn:lem2a}
&\|\zb_{s,k}-\mathsf{\Pi}_{\mathcal{X}}[\pb_{s,k}]\|^2
\le  \|\pb_{s,k}-\mathsf{\Pi}_{\mathcal{X}}[\pb_{s,k}]\|^2 \\
&~+ 2\frac{C_{\zb}(C_h+N)}{c_h^2}\|\bd_{s,k}\| +2\frac{C_h+N}{c_h^2}|\nu_{s,k}|\nonumber\\
&~ + \frac{1}{4\eta}\|\pb_{s,k} - \mathsf{\Pi}_{\mathcal{X}}[\pb_{s,k}]\|^2 + A_{\eta,\tau}\a_k^2L_f^2\nonumber\\
&~ -\frac{\tau-1}{\tau(L_h+D)^2}h_+^2(\pb_{s,k};\ccalD_{s,k},\o_{s,k}),\nonumber
\end{align}
where $A_{\eta,\tau} = 1+4\eta+\tau$ and $\eta,\tau > 0$ are arbitrary.
By the definition of the projection, we have
$
\dist(\pb_{s,k},\Xc) = \|\pb_{s,k} - \mathsf{\Pi}_{\Xc}[\pb_{s,k}]\|$
and
$
\dist(\zb_{s,k},\Xc)
= \|\zb_{s,k}-\mathsf{\Pi}_{\mathcal{X}}[\zb_{s,k}]\| \le \|\zb_{s,k}-\mathsf{\Pi}_{\mathcal{X}}[\pb_{s,k}]\|.
$
Upon substituting these estimates in relation (\ref{eqn:lem2a}), we obtain
\begin{align}\label{eqn:lem2b}
&\dist^2(\zb_{s,k},\Xc)
\le  \dist^2(\pb_{s,k},\Xc) \\
&\hspace{2cm}+ 2\frac{C_{\zb}(C_h+N)}{c_h^2}\|\bd_{s,k}\| +2\frac{C_h+N}{c_h^2}|\nu_{s,k}|\nonumber\\
&\hspace{2cm} + \frac{1}{4\eta}\dist^2(\pb_{s,k},\Xc) + A_{\eta,\tau}\a_k^2L_f^2\nonumber\\
&\hspace{2cm} -\frac{\tau-1}{\tau(L_h+D)^2}h_+^2(\pb_{s,k};\ccalD_{s,k},\o_{s,k}).\nonumber
\end{align}
Taking the expectation conditioned on $\Fc_{k-1}$ and noting that $\pb_{s,k}$ is fully determined by $\Fc_{k-1}$,
we have for any $s\in \ccalS$ and $k \ge 1$ \textit{a.s.}:
\begin{align}\label{eqn:lem2c}
&\Es\left[\dist^2(\zb_{s,k},\Xc) \mid \Fc_{k-1}\right]
\le  \dist^2(\pb_{s,k},\Xc) \\
&+ \hspace{-0.5mm}2\frac{C_{\zb}(C_h+N)}{c_h^2}\Es[\|\bd_{s,k}\|\mid \Fc_{k-1}] \hspace{-0.5mm}+\hspace{-0.5mm}2\frac{C_h+N}{c_h^2}\Es[|\nu_{s,k}|\mid \Fc_{k-1}]\nonumber\\
& + \frac{1}{4\eta}\dist^2(\pb_{s,k},\Xc)
+ A_{\eta,\tau}\a_k^2L_f^2\nonumber\\
&~ -\frac{\tau-1}{\tau(L_h+D)^2}\Es\left[h^2_+(\pb_{s,k};\ccalD_{s,k},\o_{s,k})\mid\Fc_{k-1}\right].\nonumber
\end{align}
We now choose $\tau =4$, $\eta = \k(L_h+D)^2$ and use Assumption \ref{assume:c} to yield
\begin{align*}
&\Es\left[\dist^2(\zb_{s,k},\Xc) \mid \Fc_{k-1}\right]
\le  \dist^2(\pb_{s,k},\Xc) \\
&+ \hspace{-0.5mm}2\frac{C_{\zb}(C_h+N)}{c_h^2}\Es[\|\bd_{s,k}\| \mid \Fc_{k-1}] \hspace{-0.5mm}+2\hspace{-0.5mm}\frac{C_h+N}{c_h^2}\Es[|\nu_{s,k}| \mid \Fc_{k-1}]\nonumber\\
& -\frac{1}{2\k(L_h+D)^2}\dist^2(\pb_{s,k},\Xc) + A\a_k^2L_f^2,\nonumber
\end{align*}
where $A = 5+4\k(L_h+D)^2$.
Finally, by summing over all $s \in \ccalS$ and using Lemma \ref{lem:ds} with $h(x) = \dist^2(x,\Xc)$,
we arrive at the following relation:
\begin{align} \label{eqn:lem2e}
&\sum_{s\in\ccalS}\Es\left[\dist^2(\zb_{s,k},\Xc) \mid \Fc_{k-1}\right]
\le  \sum_{s\in\ccalS}\dist^2(\zb_{s,k-1},\Xc) \\
&~+ 2\frac{C_{\zb}(C_h+N)}{c_h^2}\sum_{s\in\ccalS}\Es[\|\bd_{s,k}\|  \mid \Fc_{k-1}] \nonumber\\ &~+2\frac{C_h+N}{c_h^2}\sum_{s\in\ccalS}\Es[|\nu_{s,k}| \mid \Fc_{k-1}]\nonumber\\
&~ -\frac{1}{2\k(L_h+D)^2}\sum_{s\in\ccalS}\dist^2(\pb_{s,k},\Xc) + A\a_k^2nL_f^2.\nonumber
\end{align}
Therefore, from $\sum_{k\ge 0} \a_k^2 < \infty$ and Lemmas~\ref{lem:summable_constraints} and \ref{lem:summable_grads}, all the conditions in the convergence theorem (Theorem~\ref{thm:super}) are satisfied and the desired result follows.
\end{proof}

\begin{proof}[Proof of Lemma \ref{lem:disagree}(a)]
From Line \ref{line:p}-\ref{line:v} and \ref{line:z} of Algorithm \ref{alg:filter+opt}, we define $\eb_{s,k} \triangleq \zb_{s,k} - \pb_{s,k}$ for $s\in \ccalS$ and $k \ge 0$. Hence, $\eb_{s,k}$ can be viewed as the perturbation that we make on $\pb_{s,k}$ after the network consensus step in Line \ref{line:p} of Algorithm \ref{alg:filter+opt}.
Consider $\|\eb_{s,k}\|$, for which we can write
\begin{align} \label{eqn:lem3a}
\|\eb_{s,k}\|
\le{}& \|\zb_{s,k} - \vb_{s,k}\| +\|\vb_{s,k} - \pb_{s,k}\|\nonumber\\
= {}&\left\|\mathsf{\Pi}_{\Xc_0}\left[\vb_{s,k} - \b_{s,k}h_+'(\vb_{s,k};\ccalD_{s,k},\o_{s,k})\right] -\vb_{s,k}\right\| \nonumber\\
{}&+\left\|\mathsf{\Pi}_{\Xc_0}[\pb_{s,k}- \a_k f'(\pb_{s,k})]- \pb_{s,k}\right\|\nonumber\\
\le{}&\frac{h_+(\vb_{s,k};\ccalD_{s,k},\o_{s,k})}{c_h} +\a_kL_f,
\end{align}
where the second inequality follows from fact that $\pb_{s,k}, \vb_{s,k} \in \Xc_0$, the nonexpansiveness of the projection operator, relations \eqref{eqn:L_f1} and \eqref{eqn:dhbnd}.

Applying $(a+b)^2 \le 2a^2 + 2b^2$ in inequality (\ref{eqn:lem3a}), we have for all $s\in \ccalS$ and $k\ge 1$,
\begin{equation}\label{eqn:lem3b}
\|\eb_{s,k}\|^2 \leq 2\frac{h_+^2(\vb_{s,k};\ccalD_{s,k},\o_{s,k})}{c_h^2} +2\a_k^2L_f^2
\end{equation}
For the term $h_+^2(\vb_{s,k};\ccalD_{s,k},\o_{s,k})$ in (\ref{eqn:lem3b}), we have the following chain of relations:
\begin{align}
&h_+^2(\vb_{s,k};\ccalD_{s,k},\o_{s,k}) \nonumber\\
=&~ \left(h_+(\vb_{s,k};\ccalD_{s,k},\o_{s,k}) -h_+(\pb_{s,k};\ccalD_{s,k},\o_{s,k})\right)^2 \label{eqn:dis1}\\
&~+\hspace{-0.7mm} h_+^2(\pb_{s,k};\ccalD_{s,k},\o_{s,k})\nonumber\\
&~ + 2h_+(\pb_{s,k};\ccalD_{s,k},\o_{s,k})\nonumber\\
&~ \times\left(h_+(\vb_{s,k};\ccalD_{s,k},\o_{s,k}) -h_+(\pb_{s,k};\ccalD_{s,k},\o_{s,k})\right)\nonumber\\
\le &~ (L_h+D)^2\|\vb_{s,k}-\pb_{s,k}\|^2 + h^2_+(\pb_{s,k};\ccalD_{s,k},\o_{s,k})\label{eqn:dis2}\\
&~+ 2(L_h+D)\|\vb_{s,k}-\pb_{s,k}\|h_+(\pb_{s,k};\ccalD_{s,k},\o_{s,k})\nonumber\\
\le &~ \a_k^2(L_h+D)^2L_f^2 + h^2_+(\pb_{s,k};\ccalD_{s,k},\o_{s,k})\label{eqn:dis3}\\
&~+ 2\a_k(L_h+D)L_fh_+(\pb_{s,k};\ccalD_{s,k},\o_{s,k})\nonumber\\
\le &~ 2\a_k^2(L_h+D)^2L_f^2 + h_+^2(\pb_{s,k};\ccalD_{s,k},\o_{s,k}),\label{eqn:dis4}
\end{align}
where (\ref{eqn:dis1}) follows from adding and subtracting $h_+(\pb_{s,k};\ccalD_{s,k},\o_{s,k})$;
(\ref{eqn:dis2}) follows from relation (\ref{eqn:L_h2});
(\ref{eqn:dis3}) follows from Line \ref{line:v} of Algorithm \ref{alg:filter+opt};
(\ref{eqn:dis4}) follows from using the relation $2|a||b| \le a^2 + b^2$ with $a = \a_k(L_h+D)L_f$ and $b = h_+(\pb_{s,k};\ccalD_{s,k},\o_{s,k})$.
By combining (\ref{eqn:lem3b}) and (\ref{eqn:dis4}), we obtain
\begin{align*}
\|\eb_{s,k}\|^2 \leq &~ 2\a_k^2L_f^2\left(1+\frac{2(L_h+D)^2}{c_h^2}\right) + \frac{2h_+^2(\pb_{s,k};\ccalD_{s,k},\o_{s,k})}{c_h^2}\\
\leq &~ 2\a_k^2L_f^2\left(1+\frac{2(L_h+D)^2}{c_h^2}\right) + \frac{2L_h^2 \dist^2(\pb_{s,k},\Xc)}{c_h^2},
\end{align*}
where the last inequality is from \eqref{e:Lh}.
Therefore, from Lemma \ref{lem:feas} and $\sum_{k=1}^\infty\a_k^2<\infty$, we conclude that
$
\sum_{k=1}^\infty \|\eb_{s,k}\|^2 <\infty$ for all $ s \in \ccalS \quad a.s.$
\end{proof}

\begin{proof}[Proof of Lemma \ref{lem:disagree}(b)]
The proof of this part of Lemma will coincide with \cite[Lemma 6(b)]{lee15approx}.
\end{proof}

%% file: bios/freundlich_bio.tex
(S'13) received a B.A. in physics from Middlebury College (2010), an M.Eng. in mechanical engineering from Stevens Institute of Technology (2013), and a PhD in mechanical engineering from Duke University (2017), where he studied control theory, differential geometry, and computer vision. 
Currently, he is a Sr. Software Engineer at Tesla.
%
He currently works on cooperative indoor SLAM for mobile cameras and large-scale supply chain estimation and control.

%% file: bios/soomin_bio.tex
(S'07--M'13) is a Research Engineer at Georgia Institute of Technology.
She was formerly a Postdoctoral Associate in Mechanical Engineering and Materials Science at Duke University. 
She received her Ph.D. in Electrical and Computer Engineering from the University of Illinois, Urbana-Champaign (2013).
She received two master's degrees from the Korea Advanced Institute of Science and Technology in Electrical Engineering, and from the University of Illinois at Urbana-Champaign in Computer Science.
In 2009, she was an assistant research officer at the Advanced Digital Science Center (ADSC) in Singapore.
Her interests include theoretical optimization (convex, non-convex, online and stochastic), distributed control and optimization of dynamic networks, human-robot interactions and machine learning.

%% file: bios/zavlanos_bio.tex
(S'05--M'09) received the Diploma in mechanical engineering from the National Technical University of Athens (NTUA), Athens, Greece, in 2002, and the M.S.E. and PhD degrees in Electrical and Systems Engineering from the University of Pennsylvania, Philadelphia, PA, in 2005 and 2008, respectively. 

From 2008 to 2009 he was a Post-Doctoral Researcher in the Dept. of Electrical and Systems Engineering at the University of Pennsylvania. He then joined the Stevens Institute of Technology, Hoboken, NJ, as an Assistant Professor of Mechanical Engineering, where he remained until 2012. Currently, he is an assistant professor of Mechanical Engineering and Materials Science at Duke University, Durham, NC. He also holds a secondary appointment in the Dept. of Electrical and Computer Engineering. His research interests include a wide range of topics in the emerging discipline of networked systems, with applications in robotic, sensor, communication, and biomolecular networks. He is particularly interested in hybrid solution techniques, on the interface of control theory, distributed optimization, estimation, and networking. 

Dr. Zavlanos is a recipient of the 2014 Office of Naval Research Young Investigator Program (YIP) Award, the 2011 National Science Foundation Faculty Early Career Development (CAREER) Award, and a finalist for the Best Student Paper Award at CDC 2006.